\documentclass[11pt,oneside]{amsart}
\title{Path classification by stochastic linear recurrent neural networks}
\date{}

\author{Wiebke Bartolomaeus}
\address[Wiebke Bartolomaeus]{RWTH Aachen University\\
Chair for Mathematics of Information Processing\\ 
Pontdriesch 10\\
52062 Aachen, Germany.}
\email{wiebke.bartolomaeus@rwth-aachen.de} 

\author{Youness Boutaib}
\address[Youness Boutaib]{RWTH Aachen University\\
Chair for Mathematics of Information Processing\\ 
Pontdriesch 10\\
52062 Aachen, Germany.}
\email{boutaib@mathc.rwth-aachen.de}

\author{Sandra Nestler}
\address[Sandra Nestler]{Institute of Neuroscience and Medicine (INM-6) and Institute for Advanced Simulation (IAS-6) and JARA-Institute ``Brain Structure-Function Relationships'' (INM-10)\\
J\"ulich Research Centre\\ J\"ulich, Germany.}
\address{RWTH Aachen University\\
Aachen, Germany.}
\email{s.nestler@fz-juelich.de}   

\author{Holger Rauhut}
\address[Holger Rauhut]{RWTH Aachen University\\
Chair for Mathematics of Information Processing\\ 
Pontdriesch 10\\
52062 Aachen, Germany.}
\email{rauhut@mathc.rwth-aachen.de}

\thanks{Y. Boutaib, S. Nestler and H. Rauhut are grateful for the financial support of the Excellence Initiative of the German federal and state governments (G:(DE-82)EXS-SFneuroIC002: "Recurrence and stochasticity for neuro-inspired computation".)}

\usepackage{amssymb} % necessaire pour les mathbb 
\usepackage{amsfonts}
\usepackage{amsmath} %pour underset
\usepackage{amsthm} %pour des theoremes unnumbered
\usepackage[margin=1in]{geometry}  % set the margins to 1in on all sides
\usepackage[usenames,dvipsnames]{xcolor} % for colours
\usepackage{float} 
\usepackage{graphicx} 
\usepackage{wrapfig}
\usepackage{hyperref}
\newtheorem{theo}{Theorem}[section]
\newtheorem{Cor}[theo]{Corollary}

\newtheorem{lemma}[theo]{Lemma}
\newtheorem{notation}[theo]{Notation}
\newtheorem{Prop}[theo]{Proposition}
\newtheorem{rem}[theo]{Remark}

\numberwithin{figure}{section}

\newcommand{\inner}[2]{\left\langle #1, #2 \right\rangle}

\begin{document}	
\maketitle
	\begin{abstract} 
	We investigate the functioning of a classifying biological neural network from the perspective of statistical learning theory, modelled, in a simplified setting, as a continuous-time stochastic recurrent neural network (RNN) with identity activation function. In the purely stochastic (robust) regime, we give a generalisation error bound that holds with high probability, thus showing that the empirical risk minimiser is the best-in-class hypothesis. We show that RNNs retain a partial signature of the paths they are fed as the unique information exploited for training and classification tasks. We argue that these RNNs are easy to train and robust and back these observations with numerical experiments on both synthetic and real data. We also exhibit a trade-off phenomenon between accuracy and robustness.
	\end{abstract}
	%%%%%%%%%%%%%%%%%%%%%%%%%%%%%%%%%%%%%%%%%%%%%%%%%%%%%%%%
	%%%%%%%%%%%%%%%%%%%%%%%%%%%%%%%%%%%%%%%%%%%%%%%%%%%%%%%%
	%%
	% Introduction and motivation
	%%
	%%%%%%%%%%%%%%%%%%%%%%%%%%%%%%%%%%%%%%%%%%%%%%%%%%%%%%%%
	%%%%%%%%%%%%%%%%%%%%%%%%%%%%%%%%%%%%%%%%%%%%%%%%%%%%%%%%
	\section{Introduction}
	Recurrent neural networks (RNNs) constitute the simplest machine learning paradigm that is able to handle variable-length data sequences while tracking long term dependencies and taking into account the temporal order of the received information. These data streams appear naturally in many fields such as (audio or video) signal processing or financial data. The RNN architecture is inspired from biological neural networks where both recurrent connectivity and stochasticity in the temporal dynamics are ubiquitous. Despite the empirical success of RNNs and their many variants (long short-term memory networks (LSTMs), gated recurrent units (GRUs), etc.), several fundamental mathematical questions related to the functioning of these networks remain open: \begin{itemize}
	\item What is the exact type of information that an RNN learns from the input sequences?
	\item Training artificial RNNs with classical methods like gradient descent suffer from fundamental problems such as instability, non-convergence, exploding gradient errors \cite{Doya} and plateauing \cite{LHL}. On the other hand, biological networks seem to be robust and easy to train. How does stochasticity contribute in regard to this?
	\item What is the amount of data needed for such a network to achieve a small estimation error with high probability?
	\end{itemize}
	In the current paper, we set out to answer these questions by modelling a biological neural network as a continuous-time (stochastic) RNN with a randomly chosen connectivity matrix and an identity activation function in view of classifying data streams (in this case, time-dependent paths.) Let us say a few words about each of our three working assumptions:
	\begin{itemize}
	\item The continuous-time dynamics are a generalisation of the classical discrete-time dynamics frequently encountered in the literature \cite{GGO}. The latter can be seen as an Euler discretisation of the former as the data stream gets sampled at shorter time intervals. Working with continuous-time dynamics provides us with a richer mathematical toolbox while still being applicable to the discrete-time case and keeping key features and issues of such systems such as the dependence on the whole data sequence and its order. 
	\item Randomly generating the connectivity matrix of an RNN is the cornerstone of reservoir computing \cite{JH, MJS}. This paradigm is based on the idea that universal approximation properties can be achieved for several dynamical systems without the need to optimise all parameters and has shown exceptional performances in a variety of tasks. This working assumption has also the benefit of simplifying the training process (as will be clear from the formulas in this paper, optimising over this matrix is computationally heavy, even in the linear case.) This will consist in our case in finding a pre-processing projection vector and the parameters of a read-out map. This simplicity can be practically exploited for instance to deploy the same network to deal with several tasks (i.e. multi-tasking) without the need for heavy retraining or storing a large number of parameters, in a fashion that is reminiscent of biological networks. Compared to the existing literature (e.g. \cite{DCREA}), we included the pre-processing map (input projection vector) as a tunable parameter in order to increase the performance compared to a classical reservoir computer.
	\item We choose to work with identity activation functions in order to build the intuition as to the answer to the questions above. In this case, we obtain precise formulas. We aim to generalise the results of this study  to the non-linear case in a later work.
\end{itemize}

	Before setting out our roadmap, let us note for the sake of completeness that there exists a number of ways in which one may avoid altogether recurrent architectures in order to handle data streams and use instead a feed-forward network, which is a more studied and understood paradigm. These are usually based on the transformation of paths into fixed length vectors that can then be fed to the feed-forward structure. In particular, we cite the Independent Component Analysis \cite{HM, HH, OS}, the signature methods \cite{Graham, KO, LLN} and the PCA-type dimension reduction introduced in \cite{BHKS}. However, these methods work best when the whole signal is processed (which may be computationally heavy) before being fed to the network while RNNs are able to work with these signals in a continuous manner as they come, rendering them more suitable to real-time situations. As to the approximation properties of these recurrent architectures, there are several works which go into the direction of indicating that such properties may also hold for the path classification problem treated in this paper, although a precise statement in the stochastic case, which is of interest to us, is still missing. For example, rigorous results providing the approximation properties of (discrete-time) RNNs with a randomly generated connectivity matrix can be found in \cite{GGO2}. In \cite{FN}, it is shown that every continuous path can be approximated (in the uniform convergence norm) as the outcome of an RNN with a suitable activation function while, more recently in \cite{LHL}, the authors show that an RNN with the identity activation function can approximate any functional on a path space provided it is continuous, linear, regular and time-homogeneous.\\
	
	We will approach the problem of the binary classification of continuous-time paths with RNNs in the presence of noise from the point of view of statistical learning theory. After introducing the necessary mathematical notation and the learning setup (model, loss function, etc.) in Section \ref{sec:LearnSetup}, we will give a generalisation error bound that holds with high probability in Section \ref{Sec:GenErrorBound}. The uniform bound that we derive controls the difference between the risk of a hypothesis and its empirical counterpart and answers practical questions concerning the size of the sample and the bounds on the pre-processing and read-out maps needed to achieve a certain accuracy. Consequently, minimising the empirical risk achieves agnostic PAC learnability and gives guarantees on the ability of the empirical risk minimiser to generalise to unseen data. Section \ref{Sec:StudyERM} looks more in details into the empirical risk minimisation (ERM) procedure: 
	\begin{itemize}
		\item we compare its output to that of the popular Support Vector Machine (SVM) considered for example in \cite{NKDGRH},
		\item argue heuristically that noise, which is a natural assumption in modelling biological neural networks, provides stability and robustness against different types of perturbations to the dataset,
		\item show rigorously that in the linear case, the RNN retains a ``partial signature'' of the time-lifted input signal as global information about said signal. The empirical risk is a function of the tunable parameters of the model and the partial signatures of the training data.
	\end{itemize}
	Finally, we look into the numerical minimisation of the empirical risk using gradient descent in Section \ref{sec:num_exp}. With the explicit formulas we obtain, the global effect through time of the tunable parameters on the loss function is taken into account and we do not need some sort of unrolling of the network to apply back-propagation through time. The experiments are performed using the Japanese vowels dataset and classes of trigonometric polynomials.
	%%%%%%%%%%%%%%%%%%%%%%%%%%%%%%%%%%%%%%%%%%%%%%%%%%%%%%%%
	%%%%%%%%%%%%%%%%%%%%%%%%%%%%%%%%%%%%%%%%%%%%%%%%%%%%%%%%
	%%
	% The learning setup
	%%
	%%%%%%%%%%%%%%%%%%%%%%%%%%%%%%%%%%%%%%%%%%%%%%%%%%%%%%%%
	%%%%%%%%%%%%%%%%%%%%%%%%%%%%%%%%%%%%%%%%%%%%%%%%%%%%%%%%
	\section{The learning setup}\label{sec:LearnSetup}
	%%%%%%%%%%%
	%%
	% The recurrent network input-output map
	%%
	%%%%%%%%%%%
	\subsection{The recurrent network input-output map}
	Let $r$, $n$ and $d$ be integers. The input and the hidden state of the RNN are modelled, respectively, as $r$ and $n$-dimensional time-dependent continuous paths $x$ and $y$. Given a filtered probability space $(\Omega,\mathcal{A}, (\mathcal{F}_t)_{0 \leq t \leq T},\mathbb{P})$, a simple continuous-time model describing the time-evolution of input-output dynamics is given by the following stochastic differential equation (SDE) (which can be seen as a stochastic version of the one in \cite{SCS})
	\begin{equation}\label{eq:MainSDE}
	\mathrm{d}y(t)=(-y(t) + W\phi (y(t)) + u(x(t)))\mathrm{d}t+\Sigma \mathrm{d}B(t),\quad t\leq T.
	\end{equation}
	with initial condition taken to be $y(0)=0$. Here, $u\in \mathcal{L}(\mathbb{R}^{r}, \mathbb{R}^{n})$ is a linear pre-processing map (identified as a matrix in $\mathbb{R}^{n\times r}$), $W\in \mathbb{R}^{n\times n}$ is the network matrix that models the connection strength between neurons, $\phi$ is an activation function (applied element-wise) and $\Sigma\in \mathbb{R}^{n\times d}$ is a matrix (which we will call the noise matrix) describing the random effect of a $d$-dimensional Brownian noise $B$. In this paper, we will consider the linear case when $\phi$ is the identity function. The more interesting case where $\phi$ is non-linear will be the subject of a future work.
	
	Given a path $x$, we will denote by $y(T,x)$ (or $y_u(T,x)$ to emphasise the dependence on the pre-processing map $u$) the terminal value (i.e. at time $T$) of the solution to the equation (\ref{eq:MainSDE}). A readout map $h$ is then combined with the final hidden state of the neural network to produce a prediction $v=h(y(T,x))$. In our case, $h$ will be a labelling function, and more specifically, a hyperplane classifier.
	\begin{figure}[h!]\label{Fig:RNNArchi}
	\centering
	\includegraphics[scale=0.6]{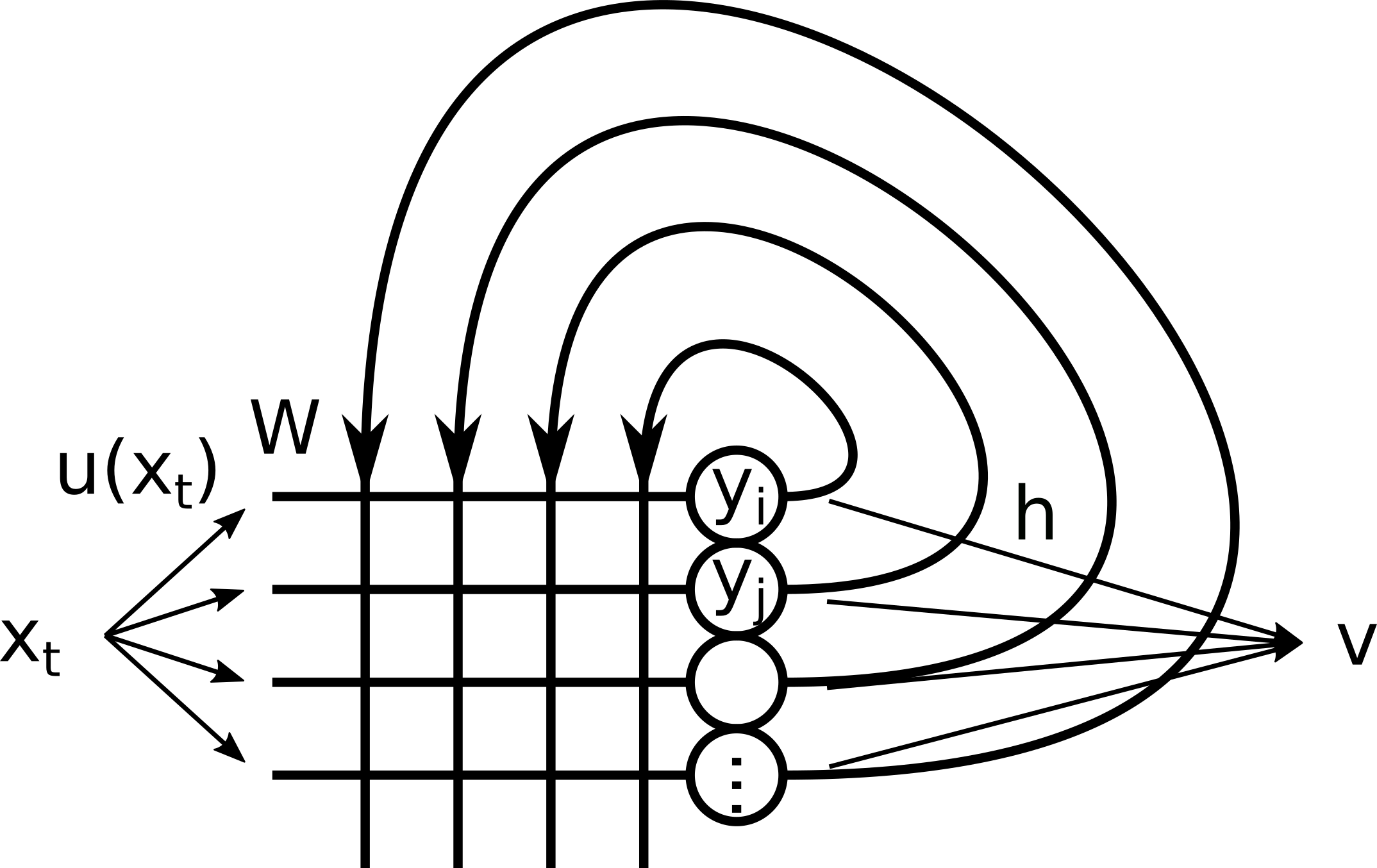}
	\caption{A sketch of the recurrent learning architecture}
	\end{figure}
	
	Given a training set of labelled inputs, we aim to train this network by changing the values of the parameters in order to increase the accuracy of future predictions (in a sense to be made clear in the following subsections). In the classical framework of reservoir computing, the network's tunable parameter is the hyperplane $h$ while the connectivity matrix $W$ and the pre-processing map $u$ are generated randomly. In our case, we aim to increase the performance by considering $u$ a tunable parameter to be optimised according to the learning task at hand.
	%%%%%%%%%%%
	%%
	% Hypothesis class and read-out maps
	%%
	%%%%%%%%%%%
	\subsection{Hypothesis class and read-out maps}
	As we have alluded to above, our global hypothesis class $\mathcal{H}$ comprises of maps that can be written as the composition of the reservoir-solution map $y_u(T,.)$ and a read-out map $h$ chosen from a read-out hypothesis class $\mathcal{H}^*$. As the read-out map $h$ will be applied to the random vector $y_u(T,x)$, we can think of the hypothesis class $\mathcal{H}$ as a class of random learners
	\[\begin{array}{cccl}
	H:& \mathcal{X} & \longrightarrow & \mathcal{V}^{\Omega},\\
		& x						& \longmapsto	& h(y_u(T,x))=v(x),\\
	\end{array}
	\]
	where $\mathcal{X}$ denotes the space of input paths and $\mathcal{V}$ the target set of outputs; for example the labels $\{-1,+1\}$ as will be in our case. If one identifies random variables and their probability distributions, then one may think of the result of the learners applied to an input $x$ as probability distributions instead of a single label. In our very simple setting, this translates into thinking of the hypothesis $H$ as a regression function	$x \longmapsto \mathbb{P}(H(x)=1)\in [0,1]$.	This discussion fits into the framework of probabilistic binary classification or in the wider one of probabilistic supervised learning as introduced in \cite{GKMO}. However, let us emphasize the fundamental difference that in our case we label inputs based on one realisation of the hypothesis rather than produce a probability distribution on the space of labels.\\
	
	In the current work, and in line with most common practices, we will take the read-out hypothesis class $\mathcal{H}^*$ to be the class of hyperplane classifiers. We will adopt the following notations:
	%%%%	
	%%%% Notation: global and hyperplane classifiers:
	%%%%
	\begin{notation}
	\begin{enumerate}
	\item For $\omega \in \mathbb{R}^n$ and $b\in \mathbb{R}$, we denote by $h_{\omega,b}$ the hyperplane classifier with normal direction $\omega \in \mathbb{R}^n$ and shift $b$
		\[ h_{\omega,b}(y)=\mathrm{sign}(\inner{y}{\omega}+b)\in\{-1,+1\}, \quad \textrm{for all } y\in \mathbb{R}^n, \] 
	with the convention $\mathrm{sign}(0)=1$.
	\item If $u\in \mathbb{R}^{n\times r}$, $H_{u,\omega,b}$ denotes the global classifier of the stochastic RNN with pre-processing map $u$
	\[H_{u,\omega,b}(x)=h_{\omega,b}(y_u(T,x)), \quad \textrm{for all } x \in \mathcal{X}. \]
	\end{enumerate}
	\end{notation}
	%%%%%%%%%%%
	%%
	% Loss function and associated risk
	%%
	%%%%%%%%%%%
	\subsection{Loss functions and associated risk}\label{subsec:LossFct}
The goal of supervised learning is to maximise the ability, with high probability and measured against a loss function, of the predicted outputs to generalise to unseen data. As our learner is generating a random variable instead of a deterministic label, we need to consider loss functions of the type $l: \mathcal{V}^{\Omega} \times  \mathcal{V} \to \mathbb{R}_+$. Given a classical loss function $\tilde{l}: \mathcal{V}\times  \mathcal{V} \to \mathbb{R}_+$ (e.g. the square loss or the binary loss), we may construct  loss functions suitable to our framework in two manners amongst others. The first way is by defining the loss function $l$ to be a statistic of the random variable $\omega\mapsto \tilde{l}(\tilde{v}(\omega),v)$, for example
	\[\begin{array}{cccl}
	l:& \mathcal{V}^{\Omega} \times  \mathcal{V} & \longrightarrow & \mathbb{R}_+,\\
		& (\tilde{v},v)					& \longmapsto	& \mathbb{E}\tilde{l} (\tilde{v},v).\\
	\end{array}
	\]
	For a fixed input and label, the sole source of randomness in our example is the $d$-dimensional Brownian motion $B$ with respect to which we will take the expectation. An explicit example would be the binary loss function $\tilde{l} (\tilde{v},v)=\mathbf{1}_{\tilde{v}\neq v}$ to which we associate the loss $l(\tilde{v},v)=\mathbb{P}(\tilde{v}\neq v)$. This is the loss function that we will consider in our case. We choose this loss function as it is simpler to analyse than other popular types of losses (square loss, hinge loss, etc.) while involving similar key quantities (the Gaussian cumulative distribution function as will be seen later) in the classification process and risk minimisation.\\
	 A second type of loss function can be obtained by defining a statistical functional $\Psi: \mathcal{V}^{\Omega} \to \mathcal{V}$ (here $\mathcal{V}$ can be understood in a broader sense, for example $\mathbb{R}$, instead of the labels $\{-1,+1\}$) then define $l(\tilde{v},v)=\tilde{l}(\Psi(\tilde{v}), v)$. This type of loss functions depends on the distribution produced by the hypothesis rather than on its single realisations and defines therefore a probabilistic loss function in the sense of \cite{GKMO}. An instance of such type would be to take $\Psi=\mathbb{E}$ and the square loss $\tilde{l} (\tilde{v},v)=|\tilde{v}-v|^2$ to obtain $l(\tilde{v},v)=|\mathbb{E}\tilde{v}- v|^2$.
	
	The generalisation error (also called risk) of a hypothesis $H\in \mathcal{H}$ associated to a loss function $l$ is then classically given by $R(H)=\mathbb{E}_{(x,v)}l(H(x),v)$, where the expectation is taken with respect to the (unknown) distribution according to which an input and its label $(x,v)$ are generated. One usually interprets the generalisation error as the ability of the learned hypothesis to generalise well for unseen data, assuming it is sampled according to the same unknown distribution which generated the training sample. As no assumptions are made in regard to the distribution of the data, the generalisation error remains unknown and thus its minimisation as such impossible. Instead, one classically aims to minimise its empirical counterpart based on a training sample $(X,V)=\{(x_i,v_i)\}_{i=1}^m$
	\[\widehat{R}_{(X,V)}(H)=\frac{1}{m} \sum_{i=1}^ml(H(x_i),v_i).\]
	Obviously, one also needs to provide theoretical arguments justifying the use of the empirical error instead of the generalised one. These often come in the form of a concentration inequality. We will recall such an argument using the notion of Rademacher complexity in a subsequent subsection.\\
	
	If we denote by $\overline{\mathcal{H}}$ the set of all measurable hypotheses, we can then decompose the difference between the true risk of a hypothesis $H\in \mathcal{H}$ and the smallest possible risk in the following manner
	\[R(H)-\inf_{H_0\in \overline{\mathcal{H}}}R(H_0)=
	\underbrace{R(H)-\inf_{H_0\in \mathcal{H}}R(H_0)}_{E_{\mathrm{est}}}  +
	\underbrace{\inf_{H_0\in \mathcal{H}}R(H_0)-\inf_{H_0\in \overline{\mathcal{H}}}R(H_0)}_{E_{\mathrm{app}}}.\]
	On the one hand, the estimation error $E_{\mathrm{est}}$ depends on the ability to solve the problem of minimising the risk over the chosen hypotheses class $\mathcal{H}$. Intuitively, this problem gets harder with larger classes of hypotheses. In many situations where the concentration inequalities mentioned above apply, one obtains quantitative guarantees on how small the estimation error $E_{\mathrm{est}}$ can be made in the case of an empirical risk minimiser 
	\begin{equation}\label{eq:ERMdef}
	H^{\mathrm{ERM}}\in 
	\underset{H\in \mathcal{H}}{\textrm{argmin }} \widehat{R}_{(X,V)}(H).
	\end{equation}
	On the other hand, the approximation error $E_{\mathrm{app}}$ depends on how accurately the hypotheses in $\mathcal{H}$ can approximate any measurable hypothesis. 
	%%%%%%%%%%%%%%%%%%%%%%%%%%%%%%%%%%%%%%%%%%%%%%%%%%%%%%%
	%%%%%%%%%%%%%%%%%%%%%%%%%%%%%%%%%%%%%%%%%%%%%%%%%%%%%%%%
	%%
	% Generalisation error bounds
	%%
	%%%%%%%%%%%%%%%%%%%%%%%%%%%%%%%%%%%%%%%%%%%%%%%%%%%%%%%%
	%%%%%%%%%%%%%%%%%%%%%%%%%%%%%%%%%%%%%%%%%%%%%%%%%%%%%%%%
	\section{Generalisation error bounds}\label{Sec:GenErrorBound}
	%%%%%%%%%%%
	%%
	% Solution to the SDE and general observations
	%%
	%%%%%%%%%%%
	\subsection{Solution to the SDE}\label{sub:GenObs}
	In our simplified case of the identity function as an activation function, and similarly to an Ornstein-Uhlenbeck process, the solution to the SDE (\ref{eq:MainSDE}) is explicitly given by
	\[y(t)= \int_0^t e^{(W-I)(t-s)}u(x(s))\mathrm{d}s+\int_0^t e^{(W-I)(t-s)}\Sigma\mathrm{d}B(s),
	\quad t\leq T.
	\]
	\begin{notation}
	For convenience, we will write $W_0=W-I$.
	\end{notation}
	The final hidden state $y(T,x)$ is then a Gaussian random vector with mean and covariance matrix given by
	\[y(T,x)\sim \mathcal{N}\left(\int_0^T e^{W_0(T-s)}u(x(s))\mathrm{d}s, \int_0^T e^{W_0(T-s)}\Sigma \Sigma^\top e^{W_0^\top (T-s)} \mathrm{d}s \right).\]
	%%%%	
	%%%% NOTATION: Mean and variance
	%%%%
	\begin{notation}
	For the rest of this paper, we will denote
	\[\nu_{x,u}=\int_0^T e^{W_0(T-s)}u(x(s))\mathrm{d}s\in \mathbb{R}^{n} \quad \textrm{and} 
	\quad A= \int_0^T e^{W_0(T-s)}\Sigma \Sigma^\top e^{W_0^\top (T-s)} \mathrm{d}s \in \mathbb{R}^{n\times n}.\]
	\end{notation}
	%%%%	
	%%%% REM: Covariance independent from parameters:
	%%%%
	\begin{rem} Note that the covariance matrix $A$ is independent from the tunable parameters $u$, $\omega$ and $b$ and the input signals. As will be seen later, this is a key property of this algorithm.
	\end{rem}
	It is also interesting to note that in this linear case, the hidden states of different inputs have similar distributions (Gaussian) differing only by their means. Loosely speaking, the role of the parameter $u$ will be then to separate as much as possible the data $x_i$ according to their labels $v_i$'s through their processed weighted averages $\nu_{x_i,u}$, while the covariance matrix $A$ quantifies by how much the hidden states are likely to deviate from their means. As in \cite{NKDGRH}, one may be tempted to apply traditional separating algorithms such as soft SVM as a classification technique. However, we follow here a strategy dictated by the risk minimisation and the learning guarantees given by the generalisation error bounds that we will obtain. We will compare these two techniques in a subsequent subsection.	
	%%%%%%%%%%%
	%%
	% Empirical risk for the binary loss
	%%
	%%%%%%%%%%%
	\subsection{Empirical risk for the binary loss}
	As stated in Subsection \ref{subsec:LossFct}, we consider the case of the binary loss function. We will denote by $\tilde{l}$ the classical binary loss function $(\tilde{v},v)\mapsto \mathbf{1}_{\tilde{v} \neq v}$, while $l$ denotes its counterpart associated with random labels (assuming measurability)
	\[l: (\tilde{v},v)\longmapsto \mathbb{E}\tilde{l}(\tilde{v},v)= \mathbb{P}\left(\tilde{v} \neq v \right).\]
	We will give first the exact formula of the loss in the pure stochastic regime, which shows the smoothing effect of the noise on the binary loss function. For a lighter notation, we will sometimes avoid making the dependence on other variables explicit. For example, we may write $y$ instead of $y_u(T,x)$.  
	%%%%	
	%%%% PROP: ERisk formula
	%%%%
	\begin{Prop}\label{Prop:EmpRiskForm} Let $(x,v)\in\mathcal{X}\times \{-1,1\}$, $u\in \mathbb{R}^{n\times r}$, $\omega \in \mathbb{R}^n$ and $b\in \mathbb{R}$. If $\omega \notin \textrm{Ker}(A)$, then
	\begin{equation}\label{def-loss}
	l(H_{u,\omega,b}(x),v)=
	\Phi\left(-v\cdot\frac{\inner{\nu_{x,u}}{\omega}+b}{\sqrt{\omega^\top A\omega}}\right),
	\end{equation}
	where $\Phi$ denotes the standard Gaussian cumulative distribution function (CDF).
	\end{Prop}
	\begin{proof}
	First note that
	\[\tilde{l}(H_{u,\omega,b}(x),v)=
	\mathbf{1}_{\mathrm{sign} (\inner{y}{\omega}+b)\neq v}
	 \leq \mathbf{1}_{v(\inner{y}{\omega}+b)\leq 0}\] 
	and that this inequality is not an identity if and only if $v=1$ and $\inner{y}{\omega}+b=0$ (because of the convention on the sign of $0$). Next, recall that $\inner{y}{\omega}+b$ is a Gaussian random variable
	\[\inner{y}{\omega}+b\sim \mathcal{N}\left(\inner{\nu_{x,u}}{\omega}+b , \omega^\top A\omega \right).\]
	We assume that $\omega \notin \textrm{Ker}(A)$. Then $\inner{y}{\omega}+b$ is a non-trivial Gaussian random variable (and $\mathbb{P}\left(\inner{y}{\omega}+b=0\right)=0$.) Consequently
	\[l(H_{u,\omega,b}(x),v)
	= \mathbb{E}_{B}\tilde{l}(H_{u,\omega,b}(x),v)
	= \mathbb{P}(v(\inner{y}{\omega}+b)\leq 0)
	=\Phi\left(-v\cdot\frac{\inner{\nu_{x,u}}{\omega}+b}{\sqrt{\omega^\top A\omega}}\right),\]
	which completes the proof.
	\end{proof}
	%%%%	
	%%%% REM: ERisk resembles distance
	%%%%
	\begin{rem} Note that if $\omega \notin \textrm{Ker}(A)$, then one cannot obtain a null empirical risk because of the Gaussian CDF. The quantity
	\[\frac{\inner{\nu_{x,u}}{\omega}+b}{\sqrt{\omega^\top A\omega}}\]
	resembles a distance from a hyperplane (more details will be provided in the next section). Similar in spirit to the soft-SVM problem, the loss function \eqref{def-loss} penalises both misclassification and close proximity to said hyperplane. \end{rem}
	%%%%	
	%%%% REM: Non-stochastic ERisk
	%%%%
	\begin{rem} If $\omega \in \textrm{Ker}(A)$, then $\inner{y}{\omega}=\inner{\nu_{x,u}}{\omega}$. We are then in a non-stochastic regime where full accuracy on the training set can be achieved if the averages $\nu_{x_i,u}$ are separable since then
	\[l(H_{u,\omega,b}(x),v)=\mathbf{1}_{\mathrm{sign}(\inner{\nu_{x,u}}{\omega}+b)\neq v}.\]
	In this regime, only misclassification is penalised.
	\end{rem}
%	%%%%	
%	%%%% REM: PAC-Bayes Loss
%	%%%%
%	\begin{rem} Note that the expression of the loss that we obtained in Proposition \ref{Prop:EmpRiskForm} is the same (in the case $A=\mathrm{I}_n$) as the one obtained by considering a Bayesian-type loss with multivariate normal prior and posterior for the classification of the terminal hidden states $\nu_{x,u}$ of the noiseless RNN (c.f. \cite{Langford, GLLM}).
	%\end{rem}
	%%%%%%%%%%%
	%%
	% Main result
	%%
	%%%%%%%%%%%
	\subsection{Main result}
	Classically, quantitative guarantees for the minimisation of the estimation error within a chosen hypothesis class (agnostic PAC-learnability) is obtained by first showing that the hypothesis set satisfies the uniform convergence property (c.f. \cite{SB}), i.e., by obtaining probabilistic bounds for the worst-in-class difference between the generalisation error and the empirical error
	\[\mathbb{P}\left(
	\sup_{H\in\mathcal{H}}\left|R(H)-\widehat{R}_{(X,V)}(H)\right|
	>\varepsilon\right).\]
	In turn, a way to achieve this is through controlling the Rademacher complexity (or the growth function or the VC-dimension) of the hypothesis class. Given a state space $\mathcal{Z}$, a sample $Z=\{z_i\}_{i=1}^m$ of points in $\mathcal{Z}$ and a class $\mathcal{G}$ of real valued maps defined on $\mathcal{Z}$, the (empirical) Rademacher complexity of $\mathcal{G}$ with respect to the sample $Z$ is defined by (c.f. \cite{MRT})
	\[\mathcal{R}_{Z}(\mathcal{G})=\frac{1}{m} \mathbb{E}_{\varepsilon}\left(\sup_{g\in \mathcal{G}} \sum_{i=1}^m \varepsilon_i g(z_i)\right),\]
	where the $\varepsilon_i$'s are independent Rademacher (i.e. symmetric Bernoulli) random variables and $\varepsilon=(\varepsilon_1,\ldots, \varepsilon_m)$\footnote{We assume that the random quantity inside the expectation is indeed measurable. This happens to be the case for most of machine learning applications including the one discussed in this paper.}. The Rademacher complexity can be seen as a measure of the richness of the class of functions $\mathcal{G}$ and its ability to provide a variety of labels $\{g(z_i)\}_{i=1}^m$ for the sample $S$. If we assume that the sample $Z$ is drawn in an i.i.d. manner according to a distribution $\mathcal{D}$, we obtain the following concentration inequality
	%%%%	
	%%%% THM: Classical Rad. ineq.
	%%%%
	\begin{theo}\label{theo:RadIneqOrig}\cite{BBL, KP, MRT} Let $\mathcal{G}$ be a family of real-valued maps defined over a sample space $\mathcal{Z}$ with values in the interval $[0,1]$. Let $\mathcal{D}$ be a distribution over $\mathcal{Z}$ and $Z=\{z_i\}_{i=1}^m\sim  \mathcal{D}^m$ be a random sample. Denote by $z$ a random variable distributed according to $\mathcal{D}$. Let $\delta>0$. Then the following holds with probability at least $1-\delta$
	\[\sup_{g\in\mathcal{G}}
	\left|\mathbb{E}g(z)-\frac{1}{m}\sum_{i=1}^mg(z_i)\right| 
	\leq 
	4 \mathcal{R}_{Z}(\mathcal{G})
	+\frac{2+5\sqrt{\log(2/\delta)/2}}{\sqrt{m}}.\]
	\end{theo}
	We will estimate the Rademacher complexity in our setting then apply the above theorem to quantify the error in estimating the true risk by its empirical counterpart. This will yield our main theoretical result below. For matrices, $\|\cdot\|$ denotes the spectral norm.
	%%%%	
	%%%% THM: Gen error bound/Rademacher/stochastic case
	%%%%
	\begin{theo}\label{thm:GenErrSto} Let $\Theta$ and $\Lambda$ be two positive real numbers. We consider the family of hypotheses given by
	\[\mathcal{H}_{\Theta, \Lambda}=\{H_{u,\omega,b}:\;\; \|\omega\|_2=1, |b|\leq \Theta, \|u\|\leq \Lambda\}.\]
	Let $\delta>0$, $R>0$ and $m\in \mathbb{N}^*$. We assume that the input signals lie almost surely in the $L^2$-ball of radius $R$
	\[\mathcal{B}_R=\left\{x\in \mathcal{C}\left([0,T],\mathbb{R}^r\right):\;\; \left(\int_0^T\left\| x_s\right\|_2^2\mathrm{d}s\right)^{1/2} \leq R\right\},\]
	and that the covariance matrix $A$ is positive definite. We denote by $\lambda_{\min}(A)$ its smallest eigenvalue. Let $\mathcal{D}$ be a distribution over $\mathcal{B}_R\times \{-1,1\}$ and $(X,V)=\{(x_i,v_i)\}_{i=1}^m\sim  \mathcal{D}^m$ be a random sample. Then, with probability at least $1-\delta$
	\[\begin{array}{l}
	\sup_{H\in\mathcal{H}_{\Theta, \Lambda}}\left|R(H)-\widehat{R}_{(X,V)}(H)\right| \\
	\leq 
	\frac{4}{\sqrt{2\pi m \lambda_{\min}(A)}} 
	\left(\Theta + 
	\Lambda R\left(\int_0^t\left\|e^{W_0(t-s)}\right\|^2\mathrm{d}s\right)^{1/2} \right)
	 +\frac{2+5\sqrt{\log(2/\delta)/2}}{\sqrt{m}}.\\
	\end{array}
	\]
	\end{theo}
	\begin{proof} By applying Theorem \ref{theo:RadIneqOrig} for the set of functions
	\[(x,v)\longmapsto l(H(x),v),\quad H\in \mathcal{H}_{\Theta, \Lambda},\]
	together with the formula obtained in Proposition \ref{Prop:EmpRiskForm} (using the assumption on $A$ being positive definite), we get that the following holds with probability at least $1-\delta$,
	\[\sup_{H\in\mathcal{H}}\left|R(H)-\widehat{R}_{(X,V)}(H)\right| \leq 
	\frac{4}{m} \mathbb{E}_{\varepsilon}\sup_{u,\omega,b} \sum_{i=1}^m \varepsilon_i \Phi\left(-v_i\cdot\frac{\inner{\nu_{x_i,u}}{\omega}+b}{\sqrt{\omega^\top A\omega}}\right)
	 +\frac{2+5\sqrt{\log(2/\delta)/2}}{\sqrt{m}},\]
	with the supremum taken over the set 
	\[\{(u,\omega,b):\;\; \|\omega\|_2=1, |b|\leq \Theta, \|u\|\leq \Lambda\}.\]
	As $\Phi$ is Lipschitz with constant $\frac{1}{\sqrt{2\pi}}$, by Talagrand's inequality \cite{LT,MRT}
	\[\begin{array}{rcl}
	\mathbb{E}_{\varepsilon}\sup\limits_{u,\omega,b} \sum\limits_{i=1}^m \varepsilon_i \Phi\left(-v_i\cdot\frac{\inner{\nu_{x_i,u}}{\omega}+b}{\sqrt{\omega^\top A\omega}}\right) &\leq &
	\frac{1}{\sqrt{2\pi}}\mathbb{E}_{\varepsilon}\sup\limits_{\omega,b,u} \sum\limits_{i=1}^m  \left(-\varepsilon_iv_i\cdot\frac{\inner{\nu_{x_i,u}}{\omega}+b}{\sqrt{\omega^\top A\omega}}\right)\\
	&=& \frac{1}{\sqrt{2\pi}}\mathbb{E}_{\varepsilon}\sup\limits_{\omega,b,u} \sum\limits_{i=1}^m  \left(\varepsilon_i\cdot\frac{\inner{\nu_{x_i,u}}{\omega}+b}{\sqrt{\omega^\top A\omega}}\right),\\
	\end{array}
	\]
	where we used that $(\varepsilon_i)\sim (\varepsilon_i v_i)$ in the last line. On the one hand
	\[\mathbb{E}_{\varepsilon}\sup\limits_{u,\omega,b} \sum_{i=1}^m  \frac{\varepsilon_i b}{\sqrt{\omega^\top A\omega}}\leq 
	\frac{\Theta}{\sqrt{\lambda_{\min}(A)}} \mathbb{E}_{\varepsilon}\left|\sum_{i=1}^m  \varepsilon_i\right|
	\leq 
	\frac{\Theta}{\sqrt{\lambda_{\min}(A)}} \sqrt{\mathbb{E}_{\varepsilon}\left(\sum_{i=1}^m  \varepsilon_i\right)^2} 
	= \frac{\Theta\sqrt{m}}{\sqrt{\lambda_{\min}(A)}}.\]
	and on the other hand
	\[\mathbb{E}_{\varepsilon}\sup\limits_{u,\omega,b}  \frac{\inner{\sum_{1}^m \varepsilon_i\nu_{x_i,u}}{\omega}}{\sqrt{\omega^\top A\omega}}
	\leq \frac{1}{\sqrt{\lambda_{\min}(A)}}\mathbb{E}_{\varepsilon}\sup_{\|u\|\leq \Lambda}\left\|\sum_{i=1}^m \varepsilon_i\nu_{x_i,u}\right\|_2.
	\]
	Recall that
	\[\sum_{i=1}^m \varepsilon_i\nu_{x_i,u}=\int_0^T e^{W_0(T-s)}\sum_{i=1}^m \varepsilon_iu(x_i(s))\mathrm{d}s.\]
	We now use the following inequality for matrix valued maps $f:[0,T]\rightarrow \mathbb{R}^{n\times n}$ and $g:[0,T]\rightarrow \mathbb{R}^n$
	\[\left\|\int_0^Tf(s)g(s)\mathrm{d}s\right\|_2\leq 
	\left(\int_0^T\left\|f(s)\right\|^2\mathrm{d}s\right)^{1/2}
	\left(\int_0^T\left\|g(s)\right\|_2^2\mathrm{d}s\right)^{1/2},\]
	together with the linearity of $u$ and $\|u\| \leq \Lambda$, ensuring $\left\|\sum_{1}^m \varepsilon_iu(x_i(s))\right\| \leq \Lambda \left\|\sum_{1}^m \varepsilon_i x_i(s)\right\|$, to obtain 
	\[\begin{array}{lcl}
	\mathbb{E}_{\varepsilon}\sup\limits_{\|u\|\leq \Lambda}\left\|\sum_{1}^m \varepsilon_i\nu_{x_i,u}\right\|_2& \leq& 
	\left(\int_0^T\left\|e^{W_0(T-s)}\right\|^2\mathrm{d}s\right)^{1/2}
	\mathbb{E}_{\varepsilon}\sup\limits_{\|u\|\leq \Lambda} \left(\int_0^T\left\|\sum_{1}^m \varepsilon_iu(x_i(s))\right\|_2^2\mathrm{d}s\right)^{1/2}\\
	& \leq& 
	\Lambda\left(\int_0^T\left\|e^{W_0(T-s)}\right\|^2\mathrm{d}s\right)^{1/2}
	\mathbb{E}_{\varepsilon} \left(\int_0^T\left\|\sum_{1}^m \varepsilon_i x_i(s)\right\|_2^2\mathrm{d}s\right)^{1/2}.\\
	\end{array}
	 \]
	Using Jensen's and Fubini's inequalities, we obtain the bound
	\[
	\mathbb{E}_{\varepsilon} \left(\int_0^T\left\|\sum_{i=1}^m \varepsilon_i x_i(s)\right\|_2^2\mathrm{d}s\right)^{1/2}
	\leq
	 \left(\int_0^T\mathbb{E}_{\varepsilon}\left\|\sum_{i=1}^m \varepsilon_i x_i(s)\right\|_2^2\mathrm{d}s\right)^{1/2}
	 =
	 \left(\sum_{i=1}^m \int_0^T\left\| x_i(s)\right\|_2^2\mathrm{d}s\right)^{1/2}.
	 \]
	 Using that the input paths live in the set $\mathcal{B}_R$, we conclude that:
	\[\mathbb{E}_{\varepsilon}\sup_{u,\omega,b} \sum_{i=1}^m \varepsilon_i \Phi\left(-v_i\cdot\frac{\mu_{x_i,\omega}+b}{\sqrt{\omega^\top A\omega}}\right)\leq
	\sqrt{\frac{m}{2\pi \lambda_{\min}(A)}}\left(
	\Theta + 
	\Lambda R\left(\int_0^T\left\|e^{W_0(T-s)}\right\|^2\mathrm{d}s\right)^{1/2}
	\right),\]
	which gives the desired inequality.
	\end{proof}
	%%%%	
	%%%% Rem: VC dimension vs. Rademacher
	%%%%
	\begin{rem}\label{Rem:VCvsRademacher} The use of the Rademacher complexity allows us to obtain a generalisation error bound that decays like $\frac{1}{\sqrt{m}}$, which is the common rate of decay encountered in the classical supervised learning framework (with i.i.d entries). This comes, however, at the cost of the technical assumption of the inputs being uniformly bounded in the $L^2$-norm (which is arguably a realistic assumption.) Similarly to the error bounds for the SVM algorithm, it could be possible to relax this assumption at the cost of a much slower rate of decay by using the notion of VC-dimension. For example, given $m$ labelled points $(X,V)=\{(x_i,v_i)\}_{i=1}^m$ in $\mathbb{R}^n$, one gets with probability at least $1-\delta$ for all hyperplane classifiers (c.f. \cite{MRT})
	\begin{equation}\label{Eq:VCdimBoundSVM}
	R(H)\leq \widehat{R}_{(X,V)}(H) + 
	\sqrt{\frac{2(n+1)\log\left(\frac{em}{n+1} \right)}{m}}+ 
	\sqrt{\frac{\log\left(1/\delta \right)}{2m}}
	.
	\end{equation}
	In the discrete case, such bounds on the VC dimension of RNNs (in function of the number of weights in the network and the length of the sequence -which would correspond here to the number of steps one uses to discretise an input path) have been obtained for example by Koiran and Sontag in \cite{KS}.
	\end{rem}
	The previous theorem allows us to quantitatively control the estimation error:
	%%%%	
	%%%% Cor: PAC Agnostic/stochastic case
	%%%%
	\begin{Cor}\label{Cor:PACAgnosticRNN} Let $\Theta$, $\Lambda$ and $R$ be positive real numbers. We assume that the input signals lie almost surely in the $L^2$-ball of radius $R$ and that the covariance matrix $A$ is definite. Then the class $\mathcal{H}_{\Theta, \Lambda}$ is agnostically PAC learnable through its empirical risk minimiser hypothesis $H^{\textrm{ERM}}$ (defined in (\ref{eq:ERMdef})), i.e., there exists a function $\tilde{m}: (0,1)^2\to \mathbb{N}$ such that for all $\varepsilon,\delta \in (0,1)$, for every distribution $\mathcal{D}$ over $\mathcal{X}\times \{-1,1\}$ and every sample $(X,V)\sim \mathcal{D}^m$ of size $m\geq \tilde{m}(\varepsilon,\delta)$, we have with probability at least $1-\delta$
	\[R(H^{\textrm{ERM}})\leq \inf_{H\in \mathcal{H}_{\Theta, \Lambda}} R(H)+\varepsilon.\]
	More explicitly, one may take $\tilde{m}(\varepsilon,\delta)$ to be an integer $m$ (ideally the smallest one) such that
	\[m\geq 
	\frac{4}{\varepsilon^2}\left(
	\frac{4 \left(\Theta + 
	\Lambda R\left(\int_0^T\left\|e^{W_0(T-s)}\right\|^2\mathrm{d}s\right)^{1/2} \right)}{\sqrt{2\pi \lambda_{\min}(A)}} 
	 +2+5\sqrt{\log(2/\delta)}\right)^2
	.\]
	\end{Cor}
	\begin{proof}
	Let $\varepsilon,\delta \in (0,1)$. Let $\tilde{m}(\varepsilon,\delta)$ be the smallest integer $m$ such that
	\[
	\frac{4}{\sqrt{2\pi m \lambda_{\min}(A)}} 
	\left(\Theta + 
	\Lambda R\left(\int_0^T\left\|e^{W_0(T-s)}\right\|^2\mathrm{d}s\right)^{1/2} \right)
	 +\frac{2+5\sqrt{\log(2/\delta)}}{\sqrt{m}}\leq \frac{\varepsilon}{2}.
	\]
	Let $\mathcal{D}$ be a distribution over $\mathcal{X}\times \{-1,1\}$ and $(X,V)=\{(x_i,v_i)\}_{i=1}^m\sim  \mathcal{D}^m$ be a random sample of size $m\geq \tilde{m}(\varepsilon,\delta)$. Note that if
	\[\sup_{H\in\mathcal{H}_{\Theta, \Lambda}}\left|R(H)-\widehat{R}_{(X,V)}(H)\right|\leq \frac{\varepsilon}{2},\]
	then for any hypothesis $H \in\mathcal{H}_{\Theta, \Lambda}$
	\[ \begin{array}{rcl}
	R(H^{\textrm{ERM}}) - R(H)&=&R(H^{\textrm{ERM}}) - \widehat{R}_{(X,V)}(H^{\textrm{ERM}}) + \widehat{R}_{(X,V)}(H^{\textrm{ERM}}) - \widehat{R}_{(X,V)}(H) \\
		&&  + \widehat{R}_{(X,V)}(H) -R(H)\\
	&\leq& \varepsilon.\\
	\end{array}
	\]
	Hence $R(H^{\textrm{ERM}})\leq \inf_{H\in \mathcal{H}_{\Theta, \Lambda}} R(H)+\varepsilon$ and
	\[\mathbb{P}\left(R(H^{\textrm{ERM}})\leq \inf_{H\in \mathcal{H}_{\Theta, \Lambda}} R(H)+\varepsilon\right)\geq 
	\mathbb{P}\left(\sup_{H\in\mathcal{H}_{\Theta, \Lambda}}\left|R(H)-\widehat{R}_{(X,V)}(H)\right|\leq \frac{\varepsilon}{2}\right).\]
	Finally, note that by Theorem \ref{thm:GenErrSto} and the definition of $\tilde{m}(\varepsilon,\delta)$ we have
	\[\mathbb{P}\left(\sup_{H\in\mathcal{H}_{\Theta, \Lambda}}\left|R(H)-\widehat{R}_{(X,V)}(H)\right|\leq \frac{\varepsilon}{2}\right)\geq 1-\delta.\]
	This completes the proof.
	\end{proof}
	%%%%	
	%%%% Rem: PAC vs PAC-Bayes
	%%%%
	\begin{rem} Corollary \ref{Cor:PACAgnosticRNN} shows an additional major advantage of the bound obtained in Theorem \ref{thm:GenErrSto}: it (quantitatively) guarantees that the empirical risk minimiser is the best-in-class hypothesis with high probability. The bounds discussed for example in Remark \ref{Rem:VCvsRademacher} aim to directly lower the risk by choosing parameters for the hypothesis that minimises the upper-bound of the risk (the right-hand side term in (\ref{Eq:VCdimBoundSVM})). The efficiency of such technique is thus very dependent on said bound being tight. 
%	If we choose this strategy, tighter bounds can usually be obtained using a PAC-Bayes approach. Assuming for example that $A=\mathrm{I}_n$, for $u\in \mathbb{R}^{n\times r}$ and any constant $C>0$, the following holds (see \cite{GLLM}) with probability at least $1-\delta$, for all $\omega \in \mathbb{R}^n$ and $b\in \mathbb{R}$ simultaneously
%	\[R(H_{u,\omega,b})\leq 
%	\frac{1}{1-e^{-C}} \left(
%	1-\exp\left(
%	-\left(
%		C\widehat{R}_{(X,V)}(H_{u,\omega,b}) + 
%			\frac{
%			\frac{1}{2}\left(\|\omega\|^2+b^2 \right)+
%\log\left(\frac{1}{\delta}\right)			
%			}{m}
%	\right)
%	\right)
%	\right).\]
%	The PAC-Bayes approach will then aim to minimise the regularised empirical risk
%	\[
%	C\widehat{R}_{(X,V)}(H_{u,\omega,b}) + 
%			\frac{
%			\|\omega\|^2+b^2		
%			}{2m}
%	\]
%	instead of the empirical risk $\widehat{R}_{(X,V)}(H_{u,\omega,b})$ suggested by a PAC approach.
	\end{rem}
	%%%%%%%%%%%%%%%%%%%%%%%%%%%%%%%%%%%%%%%%%%%%%%%%%%%%%%%
	%%%%%%%%%%%%%%%%%%%%%%%%%%%%%%%%%%%%%%%%%%%%%%%%%%%%%%%%
	%%
	% A study of the empirical risk
	%%
	%%%%%%%%%%%%%%%%%%%%%%%%%%%%%%%%%%%%%%%%%%%%%%%%%%%%%%%%
	%%%%%%%%%%%%%%%%%%%%%%%%%%%%%%%%%%%%%%%%%%%%%%%%%%%%%%%%
	\section{A study of the empirical risk}\label{Sec:StudyERM}
	In this section, we aim to provide a better understanding of the empirical risk (in view of our future work on the non-linear case). We recall that in the pure robust stochastic case (i.e., where the covariance matrix $A$ is positive definite) and given a sample $(X,V)=\{(x_i,v_i)\}_{i=1}^m$, the empirical risk of a hypothesis $H_{u,\omega,b}$ is given by the formula
	\[\widehat{R}_{(X,V)}(H_{u,\omega,b})=\frac{1}{m}\sum_{i=1}^m \Phi\left(-v_i\cdot\frac{\inner{\nu_{x_i,u}}{\omega}+b}{\sqrt{\omega^\top A\omega}}\right).\]
	In the subsequent subsections, we will compare and draw parallels between the empirical risk minimisation (ERM) and the support vector machine (SVM) approaches and show that stochastic (linear) RNNs keep a ``partial signature'' of the input path as a summary of the information about the path.
	%%%%%%%%%%%
	%%
	% An interpretation via margins
	%%
	%%%%%%%%%%%
	
	\subsection{An interpretation via margins}
	By making the change of variable $\alpha=A^{1/2}\omega$, one can rewrite
	\[\frac{\inner{\nu_{x,u}}{\omega}+b}{\sqrt{\omega^\top A\omega}}=\frac{\inner{A^{-1/2}\nu_{x,u}}{\alpha}+b}{\|\alpha\|_2}.\]
	Note that the absolute value of the quantity above is the distance of the ``transformed mean'' $A^{-1/2}\nu_{x,u}$ from the hyperplane $H(\alpha,b)$. Given such a hyperplane, let $I_0$ be the set of indices of input paths whose transformed means are correctly classified and $m_0:=|I_0|$ be its cardinal. Define the corresponding margin $\rho_0$ as the distance between the hyperplane $H(\alpha,b)$ and the closest correctly classified transformed average $A^{-1/2}\nu_{x_i,u}$,
	\[\rho_0:=\min_{ i \in I_0} \mathrm{d}(A^{-1/2}\nu_{x_i,u},H(\alpha,b)).\]
	Then, for all $i\in I_0$, one has
	\[\Phi\left(-v_i\cdot\frac{\inner{A^{-1/2}\nu_{x_i,u}}{\alpha}+b}{\|\alpha\|_2}\right)\leq \Phi(-\rho_0).\]
	Similarly, let $I_1$ be the complement of $I_0$ and $\rho_1$ the distance of the furthest misclassified average $A^{-1/2}\nu_{x_i,u}$ from the hyperplane $H(\alpha,b)$
	\[\rho_1:=\max_{i\in I_1} \mathrm{d}(A^{-1/2}\nu_{x_i,u},H(\alpha,b)).\]
	Then, for all $i\in I_1$, one has
	\[\Phi\left(-v_i\cdot\frac{\inner{A^{-1/2}\nu_{x_i,u}}{\alpha}+b}{\|\alpha\|_2}\right)\leq \Phi(\rho_1).\]
	Using these inequalities, one can bound the empirical risk as follows
	\[\widehat{R}_{(X,V)}(H) \leq \frac{m_0\Phi(-\rho_0)+(m-m_0)\Phi(\rho_1)}{m}.\]
	Intuitively, decreasing the upper bound above requires a balance between, on the one hand, correctly classifying the averages $A^{-1/2}\nu_{x_i,u}$ with a large margin $\rho_0$ and, on the other hand, misclassifying as few averages as possible with the smallest worst ``misclassification margin'' $\rho_1$. This is reminiscent of the soft SVM algorithm (see for example \cite{CV, MRT}), which is the approach adopted in \cite{NKDGRH} to classify the data according to the statistics of their classes. Let us recall that the soft SVM algorithm consists of solving the following (constrained convex) optimisation problem (the parameter $\lambda\geq 0$ is to be freely chosen depending on the desired properties of the final classifier)
	\begin{equation}\label{SVM-primal}
	\min_{\alpha,b,\xi}\frac{1}{2}\left\|\alpha\right\|_2^2+\lambda\left\|\xi\right\|_1
	\quad \textrm{subject to}\quad v_i\left(\inner{\alpha}{A^{-1/2}\nu_{x_i,u}}+b\right)\geq 1-\xi_i\;\;;\;\; \xi_i\geq 0
	;\;\; \forall 1\leq i\leq m.
	\end{equation}
	Classically, this is equivalent to the dual (quadratic) problem
	\[\max_{\theta} \sum_{1\leq i\leq m}\theta_i-\frac{1}{2}\left\|\sum_{1\leq i\leq m}{\theta}_iv_iA^{-1/2}\nu_{x_i,u}\right\|_2^2
	\quad \textrm{subject to}\quad 
	\inner{\theta}{v}=0\;\;;\;\;
	\forall i:0\leq {\theta}_i\leq \lambda.\]
	Given its minimizer $\theta^*$, the solution of the primal SVM problem \eqref{SVM-primal} can be computed as the direction $\alpha = \sum_{i\leq m}{\theta}^*_iv_iA^{-1/2}\nu_{x_i,u}$ and, given $i \in [\![1,m]\!]$ such that $0< {\theta}^*_i< \lambda$, $b=v_i -\inner{\alpha}{A^{-1/2}\nu_{x_i,u}}$.\\
	
	First, note that this optimisation problem is not convex anymore if we include the optimisation over the pre-processing map $u$. Second, while the solutions (for a fixed $u$) of the ERM and soft SVM may be similar in some \textit{generic} situations, there are some significant difference in behaviour and interpretation in the results of the two algorithms\footnote{In the following arguments, we show these differences for classification tasks for points in the plane. In our setting, this is equivalent to taking the dimensions $r=n=2$, the pre-processing map $u=\mathrm{Id}$ and the input paths as constant paths.}:
	% that make substituting one for another inadvisable
	\begin{itemize}
	\item The solution to the ERM algorithm is sensitive to the number of inputs in each class (thus, in some way, ``learning'' the data generating distribution) while that of the SVM only depends on the support vectors (Figure~\ref{Fig:DpdNbrPoints}).
	\begin{figure}[h]
	\centering
	\includegraphics[scale=0.55]{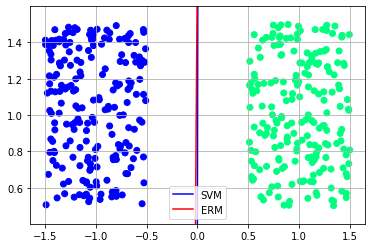}
	\includegraphics[scale=0.55]{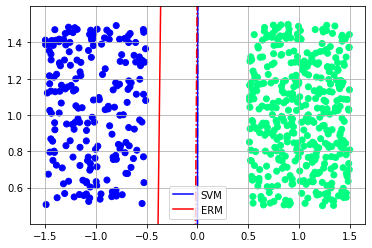}
	\caption{ERM (obtained via a gradient descent algorithm) and SVM may produce similar results in generic situations (left). However, ERM is more sensitive to the relative number of inputs in each class (right). Dotted lines on the right plot recall the position of the classifiers from the left plot.}
	\label{Fig:DpdNbrPoints}
	\end{figure}
	\item Another key difference between these two algorithms lies in their respective objectives: the ERM attempts to find a hyperplane where most of the data is (correctly classified and) far away from said hyperplane while the soft SVM attempts the same but only with respect to the support vectors. The results of these procedures can lead sometimes to very different outputs (Figure~\ref{Fig:DiffInterp}).
	\begin{figure}[h]
	\includegraphics[scale=0.75]{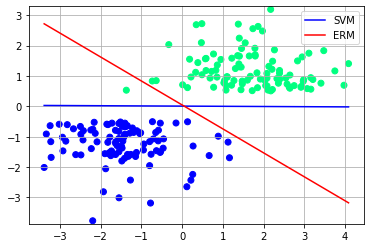}
	\caption{ERM and SVM seek two different geometric objectives that may result in very different outputs.}
	\label{Fig:DiffInterp}
	\centering
	\end{figure}
	\item The two algorithms have different sensitivities to outliers and mislabelled training data (which we will discuss in the next subsection.) 
	\item Given the preprocessing map $u$, if one denotes the hyperplane parameters returned by the SVM algorithm by $(\omega_u,b_u)$, then the smoothness of the map $u\mapsto (\omega_u,b_u)$ (and therefore that of $u\mapsto \widehat{R}_{(X,V)}(H_{u,\omega_u,b_u})$) is much less trivial to prove or even define, thus rendering the use of classical gradient descent algorithms to optimise over $u$ unjustified (indeed, attempting to do so in several numerical experiments, the algorithm failed to converge.)
	\end{itemize}
	\begin{rem} In \cite{NKDGRH}, the SVM algorithm is successfully used to separate classes that can be separated by their statistics (for example, the realisations of two Gaussian processes.) A ``good'' pre-processing map $u$ (based on a mathematical formula) is chosen beforehand; thus avoiding the use of gradient descent to optimise over this parameter.
	\end{rem}
	%%%%%%%%%%%
	%%
	% Robustness and further analysis of the ERM
	%%
	%%%%%%%%%%%
	\subsection{Robustness and further analysis of the ERM}\label{subsec:Robust_Ana_ERM}
	Making again the change of variable $\alpha=A^{1/2}\omega$, we saw from the above that, informally, an ERM algorithm has the task to find parameters $(u,\omega, b)$ such that all the transformed averages $A^{-1/2}\nu_{x_i,u}$ are correctly classified, i.e.,
	\[\forall i\in[\![1,m]\!]: \quad v_i\left( \inner{A^{-1/2}\nu_{x_i,u}}{\alpha}+b \right) \geq 0,\]
	and as distant from the hyperplane $H(\alpha, b)$ as possible (thus prompting $\Phi\left(-v\cdot\frac{\inner{\nu_{x,u}}{\omega}+b}{\sqrt{\omega^\top A \omega}}\right)$ to be as small as possible) while working under the constraint that the overall sum of the losses associated to each training input (i.e., the empirical risk) has to be as small as possible. The combination of the latter with the loss function involving the Gaussian CDF is the reason why this ERM algorithm differs from a simple classification of the transformed averages $A^{-1/2}\nu_{x_i,u}$. For example, if we assume that the two clouds
	\[C_+=\{A^{-1/2}\nu_{x_i,u}: \;\;\; v_i=1\}
	\quad \textrm{and} \quad
	C_-=\{A^{-1/2}\nu_{x_i,u}: \;\;\; v_i=-1\}\]
	are concentrated and well separated and then introduce an additional training input $x_*$ with label $v_*=1$ but such that $A^{-1/2}\nu_{x_*,u}$ is much closer to $C_-$ than to $C_+$, it is then very possible for the ERM to choose to misclassify $x_*$ rather than to correctly classify it (doing the latter might mean an increase in the empirical risk as the cloud $C_-$ moves much closer to the new hyperplane $H(\alpha, b)$.) In practice, the label $v_*$ given to the path $x_*$ could be a wrong one (i.e., mislabelled training data). However, this does not necessarily alter the result of this ERM algorithm in a drastic way (which would be the case for the hard SVM algorithm or a soft SVM algorithm with a bad choice of the regularising parameter $\lambda$). Figure \ref{Fig:MislabelSVMGD}\footnote{Here again, we illustrate our point with the classification task of planar points.} shows the dependence of the ERM and the SVM algorithms on mislabelled data.
	\begin{figure}[h]
	\includegraphics[scale=0.4]{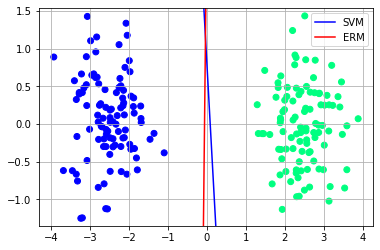}
	\includegraphics[scale=0.4]{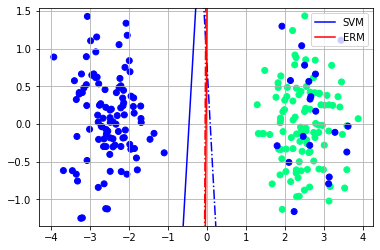}
	\includegraphics[scale=0.4]{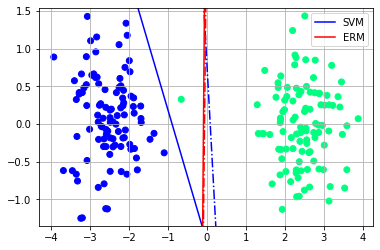}
	\centering
	\caption{The resulting hyperplanes obtained by ERM and SVM by introducing mislabelled training data (the dotted lines represent the hyperplanes obtained in the first experiment).}
	\label{Fig:MislabelSVMGD}
	\end{figure}
		
	We analyse now the bounds $\|u\|\leq \Lambda$ and $|b|\leq \Theta$ in Theorem \ref{thm:GenErrSto}. While these were necessary to obtain the bounds in said theorem, they are also necessary for the ERM to converge in general. Let us first start with the bound $\|u\|\leq \Lambda$. As far as the separation of the means $\nu_{x_i,u}$ of the processed inputs is concerned, $\frac{u}{\|u\|} $ is responsible for geometrically separating these means as well as possible, while $\|u\|$ is an amplifying factor that could be exploited for adjusting to noise and randomness in the evaluation maps once a perfect separation is achieved. For example, if $(u,\omega, b)$ are parameters of the algorithm such that all the averages $\nu_{x_i,u}$ are correctly classified in the sense that
	\[ \min_{1\leq i\leq m} v_i\left( \inner{\nu_{x_i,u}}{\omega}+b \right) \geq 0 \quad
	\textrm{and} \quad
	 \max_{1\leq i\leq m} v_i\left( \inner{\nu_{x_i,u}}{\omega}+b \right) > 0,\]
	then
	\[\lim_{\alpha \to +\infty} \widehat{R}_{(X,V)}(H_{\alpha u,\omega,\alpha b})=0.\]
	In other words, once a training algorithm finds a pre-processing map $u$ that enables a correct hyperplane separation of (the weighted means $\nu_{x,u}$ of) the data, it will tend to linearly scale the  norm of $u$ to infinity with two consequences: first, further distancing the means $\nu_{x_i,u}$ (or more precisely $A^{-1/2}\nu_{x_i,u}$ as seen earlier) from the classifying hyperplane and second, distancing these averages among themselves so that, with high probability, the random vectors $y(T,x_i)$ 's take their values in similar non-intersecting elliptic domains (getting larger with $\alpha$) centred at the $\nu_{x_i,u}$'s (since the Gaussian random vectors $y(T,x_i)$'s have the same covariance matrix $A$.) Hence, if there are no bounds on the norm of $u$ and if perfect classification is possible, the (non-converging) algorithm will try to transition from a robust classification to an accurate one (more details below.) The bound $|b|\leq \Theta$ prevents in some particular scenarios the choice of hyperplanes whose distance from the origin tends to infinity. If we take for example the case where all the data belong to the same class (less extreme examples can also be explicited), say ``$+1$'', and if $(u,\omega, b_0)$ are parameters of the algorithm such that all the averages $\nu_{x_i,u}$ are correctly classified, then
	\[\lim_{b \to +\infty} \widehat{R}_{(X,V)}(H_{u,\omega,b_0+b})=0.\]
	Note also that the bound on the norm of $u$ (and $b$) are interchangeable with a rescaling of the covariance matrix $A$ as the classifier with parameters $({\alpha u,\omega,\alpha b})$ and covariance matrix $A$ generates the same risk as the classifier with parameters $({ u,\omega, b})$ and covariance matrix $A/\alpha^2$. For this reason, we choose the bound $\Lambda=1$ in the numerical experiments and control the covariance matrix through a noise scale to be chosen carefully (see Section \ref{sec:num_exp}).\\
	
	Finally, we discuss the condition of definiteness of the covariance matrix $A$ and its role in the robustness of the classification algorithm. When $A$ is positive definite, we have seen that the goal of the ERM
	%, subject to the usual constraints on $\|u\|$ and $|b|$, 
	is solving the minimisation problem
		\begin{equation}\label{ERM-def}
		\min_{\|u\|\leq \Lambda,|b| \leq \Theta, \omega} \frac{1}{m}\sum_{i=1}^m \Phi\left(-v_i\cdot\frac{\inner{\nu_{x_i,u}}{\omega}+b}{\sqrt{\omega^\top A \omega}}\right).
		\end{equation}
	On the one hand, if the empirical risk is interpreted as a measure of the precision of the classification task on the training set, then we see that in this regime ($A$ positive definite), a full precision (null risk) is not achievable (due to the randomness of the evaluation map). The same applies to the classification of the data on a validation set as the label is a random variable whose value will depend on the current simulation (of the Brownian motion and the solution to the S.D.E). On the other hand, and as we have previously seen, note that even for a fixed pre-processing map $u$, it is clear that the choice of the optimal parameters depends explicitly on all the training data, providing in this way a robustness against mislabelled data in the training set. For these reasons and the ones detailed above, we call this the robust or the stochastic regime.\\
	If we now take $A=0$, the objective of an ERM becomes the solution of the minimisation problem
	\[
	\min_{\|u\|\leq \Lambda,|b| \leq \Theta, \omega} \frac{1}{m}\sum_{i=1}^m \mathbf{1}_{\mathrm{sign}\left(\inner{\nu_{x_i,u}}{\omega}+b\right)\neq v_i}.
	\]
	The task at hand is a mere classification of the averages $\nu_{x_i,u}$ of the processed data. If there exists a pre-processing map $u$ such that these averages are separable, then an ERM algorithm will achieve a full precision on the training data set (and in general, a unique solution does not exist unless we introduce an additional constraint like maximising the margin of the classifying hyperplane as in SVM). However, this classification will generally depend only on some support vectors $\nu_{x_i,u}$ (i.e. the closest averages to the classifying hyperplane) while ignoring the information provided by the rest of the training data, hence potentially becoming sensitive and vulnerable to mislabelling and the data-generating distribution. We call such regime the accurate regime.\\
	In the general case ($A$ not positive definite but not necessarily null), the ERM may choose either the robust or the accurate regime depending on how well the data can be separated. If the averages $\nu_{x_i,u}$ are separable, the accurate regime is preferred as it leads to a null empirical risk, otherwise the ERM may choose a robust solution. We refer for example to \cite{FFF, RXYDL, TSETM} for more information on the general topic of the trade-off between accuracy and robustness.
	%%%%%%%%%%%
	%%
	% Information retained by stochastic RNNs
	%%
	%%%%%%%%%%%
	\subsection{Information retained by stochastic RNNs}
	In this subsection, we highlight that the ERM is equivalent to a minimisation of a functional of the signatures of the augmented input paths $(t,\int x_t\mathrm{d}t)_{t\leq T}$. We recall that the signature $S(z)=(Z^n)_{n\in\mathbb{N}}$ of a path $z$ defined over an interval $[0,T]$ with values in a Banach space $E$ is the sequence of its iterated integrals, i.e., for all $s\leq t$, 
	\[\left\{\begin{array}{rcl}
	Z^0{(s,t)}&=&1,\\
	Z^{1}{(s,t)}&=&\int_{s \leq u\leq t}\mathrm{d}z(u)=z(t)-z(s) \in E,\\
	Z^{n+1}{(s,t)}&=&\int_{s \leq u\leq t}Z^{n}{(s,u)}\otimes \mathrm{d}z(u) \in E^{\otimes (n+1)} \;\;\; \textrm{for all } n\in \mathbb{N}. \\
	\end{array}\right.\]
	The role of the signature in RNNs should not come as a surprise for two main reasons: 
	\begin{itemize}
		\item The signature of a path is the only information needed from a path to solve a differential equation. This is the cornerstone of the theory of rough paths (c.f. \cite{Lyons, CLL}.)
		\item Every path is uniquely characterised (up to what is called a tree-like equivalence) by its signature \cite{HL}. This fact is the basis for a lot of research in the machine learning of data streams. These streams are mapped via their (truncated) signatures to vectors in the tensor algebra, thus allowing one to use many of the classical machine learning techniques (like feed-forward neural networks) for data streams (e.g. \cite{Graham, KO, LLN}).
	\end{itemize}
	For more information, we refer the reader to the original work of K.T. Chen \cite{Chen}, the modern formulation in \cite{CLL} and the introduction to signature methods in machine learning in \cite{CK}.\\
	
	We now reformulate the ERM problem in a way that fully separates the tunable parameters of the network from the information provided by the paths in the training dataset.
	%%%%	
	%%%% THM: kernel formulation
	%%%%
	\begin{theo}\label{theo:kernel_sig} Define $\mathbb{H}$ as the Hilbert space
	\[
	\mathbb{H}:=
	\left\{\left(a,(b_k)_{k\geq 0}\right)\in \mathbb{R}\oplus\bigoplus_{k\geq 0} \mathbb{R}^r:\;\;\; 
	\sum_{k\geq 0}k! \|b_k\|_2^2 <\infty	
	\right\},\]
	with inner product given by
	\[
	\inner{\left(a,(b_k)_{k\geq 0}\right)}{\left(c,(d_k)_{k\geq 0}\right)}_\mathbb{H}=ac+\sum_{k\geq 0} k!b_k^\top d_k\quad
	\textrm{for all } \left(a,(b_k)_{k\geq 0}\right), \left(c,(d_k)_{k\geq 0}\right) \in \mathbb{H}.\]
	 Then the ERM is equivalent to the following optimisation problem
	 \[\min_{(\beta,(\theta_k)_{k\geq 0})\in \mathcal{C}_A}\sum_{i=1}^m\Phi\left(-v_i
	\inner{\left(\beta,\left(\theta_k\right)_{k\geq 0}\right)}
	{\left(1,\left(\int_0^T \frac{(T-s)^k}{k!}x(s)\mathrm{d}s\right)_{k\geq 0}\right)}_{\mathbb{H}} 
	 \right),\]
	where $\mathcal{C}_A$ is defined as
	%the cone in $\mathbb{H}$ defined by
	\begin{align*}
	& \mathcal{C}_A  := \mathcal{C}_{A,\Lambda,\Theta} \\
	&:= 
	\left\{
	(\beta,(\theta_k)_{k\geq 0})\in \mathbb{H}: \beta \in \mathbb{R}, \theta_k=u^{\top}\frac{(W_0^\top)^k}{k!}\alpha,
	 u \in \mathbb{R}^{n\times r}, \|u\|\leq \Lambda, |\beta| \leq \frac{\Theta}{\sqrt{\lambda_{\min}(A)}}, \alpha \in \mathbb{R}^n, \alpha^\top A\alpha=1
	\right\}.
	\end{align*}
	\end{theo}
	\begin{proof}
	We write the problem at hand 
	 \[
	 \min_{\|u\|\leq \Lambda,|b| \leq \Theta, \omega} \sum_{i=1}^m\Phi\left(-v_i\cdot\frac{\inner{\omega}{\nu_{x_i,u}}+b}{\sqrt{\omega^\top A \omega}}\right)
	\quad\textrm{s.t.}\quad
	\omega\neq 0,
	\]
	in the simpler equivalent form
	\[\min_{\|u\|\leq \Lambda, |\beta| \leq \frac{\Theta}{\sqrt{\lambda_{\min}(A)}}, \omega} \sum_{i=1}^m\Phi\left(-v_i(\inner{\alpha}{\nu_{x_i,u}}+\beta )\right)
	\quad\textrm{s.t.}\quad
	\alpha^\top A\alpha=1.\]
	Note now that we can expand the mean $\nu_{x,u}$ into a series
	\[\nu_{x,u}=\int_0^T e^{W_0(T-s)}u(x(s))\mathrm{d}s
	=\sum_{k=0}^{\infty}k!\frac{W_0^k}{k!}u\left(\int_0^T \frac{(T-s)^k}{k!}x(s)\mathrm{d}s\right),\]
	so that one can separate, in the inner product $\inner{\alpha}{\nu_{x,u}}+\beta$, the tunable parameters from a functional of the input signals
	\[\begin{array}{ccl}
	\inner{\alpha}{\nu_{x,u}}+\beta&=&
	\sum\limits_{k=0}^{\infty}k!\alpha^{\top}\frac{W_0^k}{k!}u\left(\int_0^T \frac{(T-s)^k}{k!}x(s)\mathrm{d}s\right)+\beta\\
	&=&	\inner{\left(\beta,\left(u^{\top}\frac{(W_0^\top)^k}{k!}\alpha\right)_{k\geq 0}\right)}
	{\left(1,\left(\int_0^T \frac{(T-s)^k}{k!}x(s)\mathrm{d}s\right)_{k\geq 0}\right)}_{\mathbb{H}}.\\
	\end{array}
	\]
	This concludes the proof.
	\end{proof}
	%%%%	
	%%%% Rem: The sequence looks like a signature
	%%%%
	\begin{rem}\label{Rem:SigKerRNN} Note that
	\[\int_0^T \frac{(T-s)^k}{k!}x(s)\mathrm{d}s=
	\int_{0<s<u_1<\cdots<u_k<T} \mathrm{d}X(s)\mathrm{d}u_1\ldots\mathrm{d}u_k,
	\]	
	where $X_{\boldsymbol{\cdot}}:=\int_{0}^{\boldsymbol{\cdot}} x(s)\mathrm{d}s$ denotes a primitive of $x$. Hence the sequence $\left(\int_0^T \frac{(T-s)^k}{k!}x(s)\mathrm{d}s\right)_{k\geq 0}$ can be obtained from the signature of the path $(t,X_t)_{t\leq T}$. 
	\end{rem}
	Theorem \ref{theo:kernel_sig} highlights the information retained by a stochastic RNN (with architecture and loss function chosen as described in this paper.) More explicitly, we have the following theorem.
	%%%%	
	%%%% THM: info retained by RNNs
	%%%%
	\begin{theo}\label{theo:RNN_info} Both during training (solving the ERM problem) and classification (generating labels for unseen data), the continuous-time stochastic RNN (with linear activation function) is uniquely determined by
	the partial signature $\widehat{S}$ of the input path defined by 
	\begin{equation}\label{def-partial-signature}
	\widehat{S}(x):=\left(\int_0^T \frac{(T-s)^k}{k!}x(s)\mathrm{d}s\right)_{k\geq 0}.
	\end{equation}
	\end{theo}
	\begin{proof}
	We have already shown in Theorem \ref{theo:kernel_sig} that the ERM problem can be written as a minimisation problem of a functional of the partial signatures of the training paths (and that no other information about said paths is required). If we denote by $(u,\omega,b)$ some possible parameters of an ERM classifier $H^{\textrm{ERM}}$ and given an input path $x$, then the RNN generates the label $v=\mathrm{sign}(\inner{y_u(T,x)}{\omega}+b)$ with $y_u(T,x)\sim \mathcal{N}\left(\nu_{x,u}, A \right)$. Following the proof of Theorem \ref{theo:kernel_sig}, $\nu_{x,u}$ can also be expressed as a functional of $\widehat{S}(x)$ only
	\[\nu_{x,u}=\sum_{k=0}^{\infty}{W_0^k}u\left(\int_0^T \frac{(T-s)^k}{k!}x(s)\mathrm{d}s\right).\]
	This completes the proof of the claims.	
	\end{proof}
	We believe that a similar result as Theorem \ref{theo:RNN_info} holds for generic RNNs, i.e., that continuous-time RNNs (even with non-linear activation functions) can be viewed as kernel machines involving the (full) signatures of the input paths. We refer for example to \cite{Lim} (and the references therein) for first results in this direction.\\
	
	Even though the RNN uses only the partial signature \eqref{def-partial-signature} of the time-lifted input paths (instead of the full signature), it turns out that it is still a faithful representation of continuous paths.
	%%%%	
	%%%% THM: partial Signatures are faithful
	%%%%
	\begin{theo}\label{theo:Sig_faithful} The partial signature map $\widehat{S}$ defined over the space of continuous paths defined over an interval $[0,T]$ and with values in a finite-dimensional space is injective.
	\end{theo}
	\begin{proof}
	Without loss of generality, we will consider the case of real-valued paths. By a simple change of variable and a reparametrisation of the path, it is equivalent to show that the map
	\[\widetilde{S}(x):\quad x\longmapsto \left(\int_0^T s^kx(s)\mathrm{d}s\right)_{k\geq 0}\]
	is injective. As $\widetilde{S}$ is linear it is also equivalent to show that $\widetilde{S}(x)=0$ if and only if $x=0$. Let then $x$ be a continuous real-valued path such that $\widetilde{S}(x)=0$. Then for every polynomial function $P$, one has $\int_0^T P(s)x(s)\mathrm{d}s=0$. Let $\varepsilon>0$ be arbitrary and let $P$ be a polynomial function such that $\|x-P\|_{\infty,[0,T]}\leq \varepsilon$. Then one has
	\[\left|\int_0^T x^2(s)\mathrm{d}s \right|
	= \left|\int_0^T x(s)(x(s)-P(s))\mathrm{d}s \right|
	\leq \varepsilon \|x\|_{\infty,[0,T]} T.
	\]
	Therefore $\int_0^T x^2(s)\mathrm{d}s=0$, from which we conclude that indeed $x=0$.
	\end{proof}
	%%%%	
	%%%% Rem: Non-distinguishable signatures
	%%%%
	\begin{rem} Despite the partial signature transform being injective, it is still possible for the linear RNN not to be able to distinguish two paths $x$ and $\widetilde{x}$ if the inner product of their partial signatures with the matrices $\left({W_0^k}u\right)_{k\geq 0}$ are the same, i.e.,
	\[\nu_{x,u}=\sum_{k=0}^{\infty}{W_0^k}u\left(\int_0^T \frac{(T-s)^k}{k!}x(s)\mathrm{d}s\right)=
	\sum_{k=0}^{\infty}{W_0^k}u\left(\int_0^T \frac{(T-s)^k}{k!}\widetilde{x}(s)\mathrm{d}s\right)=\nu_{\widetilde{x},u}.\]
	Replacing the sequence $\left({W_0^k}u\right)_{k\geq 0}$ by another whose entries can be independent is the cornerstone and the reason behind the power of signature techniques. However, these techniques are confronted with computational issues and therefore the signature has to be restricted to low orders. RNNs are able however to surmount this obstacle by increasing the dimension $n$ and thus allowing for more degrees of freedom (while signatures are computed implicitly as shown above through the recursive architecture.)
	\end{rem}
	
	Using the factorial decay of the elements of the partial signature sequence (which is trivial and explicit in our case), we can then obtain a good approximation for the ERM by truncating the signature. To show such result, we first prove the following technical lemma.
	%%%%	
	%%%% Lemma: Factorial decay control
	%%%%
	\begin{lemma}\label{lemma:factorial_decay} Let $N\in \mathbb{N}^*$, $u\in \mathbb{R}^{n\times r}$, $\omega \in \mathbb{R}^n$ and $x\in L^1 \left([0,T],(\mathbb{R}^r,\|\cdot\|_2)\right)$. Then
	\[\left| \inner{\omega}{\nu_{x,u}}- 
	\sum\limits_{k=0}^{N}\omega^{\top}{W_0^k}u\left(\int_0^T \frac{(T-s)^k}{k!}x(s)\mathrm{d}s\right)
	\right|
	\leq 
	\left\|\omega\right\|_2
	\left\|u\right\|
	e^{\left\|W_0\right\| T} \frac{\left(\left\|W_0\right\| T\right)^{N+1}}{(N+1)!}
	\int_0^T \left\|x(s)\right\|_2 \mathrm{d}s.
	\]
	\end{lemma}
	\begin{proof}
	Expanding the difference that we want to bound we get
	\[
	\inner{\omega}{\nu_{x,u}}- 
	\sum\limits_{k=0}^{N}\omega^{\top}{W_0^k}u\left(\int_0^T \frac{(T-s)^k}{k!}x(s)\mathrm{d}s\right)
	=\sum\limits_{k=N+1}^{\infty}\omega^{\top}{W_0^k}u\left(\int_0^T \frac{(T-s)^k}{k!}x(s)\mathrm{d}s\right).
	\]
	For every $k\in \mathbb{N}$, one trivially has
	\[
	\left|\omega^{\top}{W_0^k}u\left(\int_0^T \frac{(T-s)^k}{k!}x(s)\mathrm{d}s\right)
	\right| \leq 
	\left\|\omega\right\|_2
	\left\|W_0\right\|^k
	\left\|u\right\|
	\frac{T^k}{k!}\int_0^T \left\|x(s)\right\|_2 \mathrm{d}s,
	\]
	while for every $a\geq 0$
	\[
	\left|\sum\limits_{k=N+1}^{\infty}\frac{a^k}{k!}\right| \leq e^a \frac{a^{N+1}}{(N+1)!}.
	\]
	The combination of the three arguments above gives then the desired result.
	\end{proof}
	We can now show how one may exploit the signature-based expansion of the empirical risk in order to obtain an approximate solution to the ERM.
	%%%%	
	%%%% THM: Truncation of signature for optim
	%%%%
	\begin{theo}\label{theo:optim_sig_trunc} Let $\Theta$, $\Lambda$ and $R$ be positive real numbers and $N\in \mathbb{N}^*$. We consider the set of parameters
	\[\mathcal{P}_{\Theta, \Lambda}=\{(u,\omega,b):\;\; \|\omega\|_2=1, |b|\leq \Theta, \|u\|\leq \Lambda\}.\]
	Assume that all input paths take their values in the $L^1$-ball of radius $R$
	\[\mathcal{B}_R=\left\{x\in \mathcal{C}\left([0,T],\mathbb{R}^r\right):\;\; \int_0^T\left\| x(s)\right\|_2\mathrm{d}s \leq R\right\}.\]
	Given a sample $(X,V)=\{(x_i,v_i)\}_{i=1}^m$, let $(u_0,\omega_0,b_0)$ be a solution to the ERM problem 
	\[\min_{(u,\omega, b) \in \mathcal{P}_{\Theta, \Lambda}} \frac{1}{m}\sum_{i=1}^m \Phi\left(-v_i\cdot\frac{\inner{\nu_{x_i,u}}{\omega}+b}{\sqrt{\omega^\top A \omega}}\right),\]
	and $(\bar{u},\bar{\omega},\bar{b})$ be a solution to the ``truncated'' ERM problem 
	\[\min_{(u,\omega, b) \in \mathcal{P}_{\Theta, \Lambda}} \frac{1}{m}\sum_{i=1}^m 
	\Phi\left(-v_i\cdot
	\frac{\sum\limits_{k=0}^{N}\omega^{\top}{W_0^k}u\left(\int_0^T \frac{(T-s)^k}{k!}x(s)\mathrm{d}s\right)+b}{\sqrt{\omega^\top A \omega}}
	\right).\]
	Then
	\[0
	\leq \widehat{R}_{(X,V)}(H_{\bar{u},\bar{\omega},\bar{b}}) - \widehat{R}_{(X,V)}(H_{u_0,\omega_0,b_0})
	\leq 
	\Lambda R e^{\left\|W_0\right\| T} \sqrt{\frac{2}{\lambda_{\min}(A)\pi}}  \frac{\left(\left\|W_0\right\| T\right)^{N+1}}{(N+1)!}.\]
	\end{theo}
	\begin{proof} For lighter expressions, we will introduce the following notation
	\[\widehat{R}_{(X,V)}^N(H_{u,\omega,b}) =\frac{1}{m}\sum_{i=1}^m 
	\Phi\left(-v_i\cdot
	\frac{\sum\limits_{k=0}^{N}\omega^{\top}{W_0^k}u\left(\int_0^T \frac{(T-s)^k}{k!}x(s)\mathrm{d}s\right)+b}{\sqrt{\omega^\top A \omega}}
	\right).
	\]
	The inequality $\widehat{R}_{(X,V)}(H_{u_0,\omega_0,b_0}) \leq \widehat{R}_{(X,V)}(\bar{u},\bar{\omega},\bar{b})$ is a direct consequence of the definition of $(u_0,\omega_0,b_0)$. We decompose the difference of these two terms in the following way
	\[
	\begin{array}{rcl}
	\widehat{R}_{(X,V)}(H_{\bar{u},\bar{\omega},\bar{b}})-\widehat{R}_{(X,V)}(H_{u_0,\omega_0,b_0})
	&=&
	\widehat{R}_{(X,V)}(H_{\bar{u},\bar{\omega},\bar{b}})-\widehat{R}_{(X,V)}^N(H_{\bar{u},\bar{\omega},\bar{b}}) \\
	&+&\widehat{R}_{(X,V)}^N(H_{\bar{u},\bar{\omega},\bar{b}})-\widehat{R}_{(X,V)}^N(H_{u_0,\omega_0,b_0})\\
	&+&\widehat{R}_{(X,V)}^N(H_{u_0,\omega_0,b_0})-\widehat{R}_{(X,V)}(H_{u_0,\omega_0,b_0}).\\
	\end{array} 
	\]
	By the definition of $(\bar{u},\bar{\omega},\bar{b})$, the second difference is non-positive,
	\[ \widehat{R}_{(X,V)}^N(H_{\bar{u},\bar{\omega},\bar{b}})-\widehat{R}_{(X,V)}^N(H_{u_0,\omega_0,b_0})\leq 0.\]
	Let $(u,\omega,b) \in \mathcal{P}_{\Theta, \Lambda}$. As $\Phi$ is $\frac{1}{\sqrt{2\pi}}$-Lipschitz, we have
	\[
	\left| \widehat{R}_{(X,V)}(H_{u,\omega,b}) - \widehat{R}_{(X,V)}^N(H_{u,\omega,b})\right|
	\leq \frac{1}{\sqrt{2\pi} m} \sum_{i=1}^m 
	\left|\frac{\inner{\nu_{x_i,u}}{\omega}- \sum\limits_{k=0}^{N}\omega^{\top}{W_0^k}u\left(\int_0^T \frac{(T-s)^k}{k!}x(s)\mathrm{d}s\right)}{\sqrt{\omega^\top A \omega}}\right|.\]
	For each $i\in [\![1,m]\!]$, the following holds by Lemma \ref{lemma:factorial_decay} and the assumptions on the parameters $\omega$ and $u$ and the path $x_i$
	\[\left|{\inner{\nu_{x_i,u}}{\omega}- \sum\limits_{k=0}^{N}\omega^{\top}{W_0^k}u\left(\int_0^T \frac{(T-s)^k}{k!}x(s)\mathrm{d}s\right)}\right|\leq \Lambda R
	e^{\left\|W_0\right\| T} \frac{\left(\left\|W_0\right\| T\right)^{N+1}}{(N+1)!}.\]
	The result is then directly obtained by applying the above bounds to the vectors $(\bar{u},\bar{\omega},\bar{b})$ and $(u_0,\omega_0,b_0)$.
	\end{proof}	
	Theorem \ref{theo:optim_sig_trunc} demonstrates then the possible power of the application of the signature-based decomposition of the empirical risk. In our setting, this technique avoids for example the expensive computation of $\nu_{x,u}$ (as integrals of time-dependent matrix exponentials) for each path in the training set and replaces it with the computation and storage of the powers $(W_0^k)_{k\leq N}$. However, in this still simple setting, we will not base our optimisation techniques on this method and will reserve its application for the non-linear case where exact formulae are not available.
	%%%%%%%%%%%%%%%%%%%%%%%%%%%%%%%%%%%%%%%%%%%%%%%%%%%%%%%
	%%%%%%%%%%%%%%%%%%%%%%%%%%%%%%%%%%%%%%%%%%%%%%%%%%%%%%%%
	%%
	% Numerical results
	%%
	%%%%%%%%%%%%%%%%%%%%%%%%%%%%%%%%%%%%%%%%%%%%%%%%%%%%%%%%
	%%%%%%%%%%%%%%%%%%%%%%%%%%%%%%%%%%%%%%%%%%%%%%%%%%%%%%%%
	\section{Numerical results}\label{sec:num_exp}	
	In this subsection, we present the results of some numerical experiments run on real and synthetic data. The focus will be on the effect of noise on the accuracy and the robustness of the RNN and on the verification of the theoretical bound obtained in Theorem \ref{thm:GenErrSto}.\\
	
	The application to real-world data will be demonstrated on the Japanese Vowels dataset\footnote{\href{https://archive.ics.uci.edu/ml/datasets/Japanese+Vowels}{ https://archive.ics.uci.edu/ml/datasets/Japanese+Vowels}}. This dataset contains speech recordings of nine male subjects pronouncing a combination of Japanese vowels. The recordings are in the form of $12$-dimensional ($r=12$) discrete time series with varying lengths. To create a binary classification problem of continuous time paths, we restrict ourselves to the recordings of the first two subjects and we reparametrise the time-series to the unit interval.\\
	For the synthetic data, we consider $5$-dimensional trigonometric polynomials ($r=5$)
	\[x(t)=\sum_{k=0}^N a_k \cos(kt)+\sum_{k=1}^N b_k \sin(kt)\]
	defined over the interval $[0,2\pi]$. Here, we take $N=6$ and the parameters $a_k$ and $b_k$ are chosen (uniformly randomly) in the intervals $I^r$ and $J^r$ respectively, where, for one class $I = [-0.2, 1]$ and $J = [-1, 0.2]$ and for the other  $I = [-1, 0.2]$ and $J = [-0.2, 1]$. We generate 70 samples for each class for our dataset.\\
	
	We take the dimension of the network to be $n=50$ (and the dimension $d$ of the Brownian motion is taken to be equal to $n$.) This is quite large for our simple synthetic dataset and appropriate for the Japanese Vowels dataset but still small compared to typical reservoirs where $n$ can be larger than $500$ to make the reservoir self-averaging. The entries of the connectivity matrix $W$ are drawn from independent Gaussian distributions $\mathcal{N}(0, \frac{0.9}{\sqrt{n}})$ (the choice of the value 0.9 is to insure the numerical stability of the system). To generate the noise matrix $\Sigma$, we first draw positive values $\lambda_1,\ldots, \lambda_n$ uniformly over $(0,1)$, then uniformly a random orthogonal matrix $U$ to output the matrix
	\begin{equation}\label{Eq:NoiseMatrixFormula}
	\Sigma = \delta \cdot U^\top \mathrm{diag}(\lambda_1,\ldots, \lambda_n) U.
	\end{equation}
	The parameter $\delta$ is a noise scale that has to be chosen carefully. A large value for $\delta$ makes the noise dominant and the RNN unable to reach an acceptable accuracy, while a too small value makes the system less robust.\\
	
	The ERM problem will be solved numerically using a gradient descent (GD) algorithm. Recall that, given a training sample $(X,V)=\{(x_i,v_i)\}_{i=1}^m$, we aim to minimise the quantity
	\[\frac{1}{m}\sum_{i=1}^m \Phi\left(-v_i\cdot\frac{\inner{\nu_{x_i,u}}{\omega}+b}{\sqrt{\omega^\top A \omega}}\right),\]
	as a function of the parameters $u$, $\omega$ and $b$. As discussed in Subsection \ref{subsec:Robust_Ana_ERM}, we impose the bound $\|u\|\leq 1$. To achieve a superior classification performance, we do not impose an a priori bound on $b$. Let us first say a few words about some simplifications in the implementation of the gradient descent (GD) algorithm in this simple linear case. If we consider the loss function associated with a single observation $(x,v)$ as a function of $u$, $\omega$ and $b$
	\[L(u,\omega,b)=\Phi\left(-v\frac{\inner{\nu_{x,u}}{\omega}+b}{\sqrt{\omega^\top A \omega}}\right),\]
	then we may explicitly compute the gradient of $L$ (for $i\in [\![1,n]\!]$ and $j\in [\![1,r]\!]$)
	\[\left\{\begin{array}{lcl}
	\frac{\partial L}{\partial u_{i,j}}(u,\omega,b)&=&-\frac{v}{\sqrt{2\pi}\sqrt{\omega^\top A \omega}}\exp\left(-\frac{1}{2}\left(\frac{\inner{\nu_{x,u}}{\omega}+b}{\sqrt{\omega^\top A \omega}}\right)^2 \right)\inner{\omega}{\nu_{x,e_{i,j}}},\\
	\frac{\partial L}{\partial \omega_i}(u,\omega,b)&=&
	-\frac{v}{\sqrt{2\pi}\sqrt{\omega^\top A \omega}}
	\exp\left(-\frac{1}{2}\left(\frac{\inner{\nu_{x,u}}{\omega}+b}{\sqrt{\omega^\top A \omega}}\right)^2 \right)
	 \left((\nu_{x,u})_i-\frac{(A\omega)_i(\inner{\omega}{\nu_{x,u}}+b)}{{\omega^\top A \omega}} \right),
	\\
	\frac{\partial L}{\partial b}(u,\omega,b)&=&-\frac{v}{\sqrt{2\pi}\sqrt{\omega^\top A \omega}}\exp\left(-\frac{1}{2}\left(\frac{\inner{\nu_{x,u}}{\omega}+b}{\sqrt{\omega^\top A \omega}}\right)^2 \right),\\
	\end{array}\right.
	\]
	where $e_{i,j}$ stands for the canonical $n\times r$ basis matrix $\left(\delta_{ik}\delta_{jl}\right)_{k\leq n, l\leq r}$ and $z_i$ stands for the $i$\textsuperscript{th} coordinate of the vector $z$. Apart from its simple expression, one can see that the averages $(\nu_{x,e_{i,j}})_{i\leq n, j\leq r}$ need to be computed for each example $x$ in the training set in order to run the algorithm. One can then exploit these to compute the mean $\nu_{x,u}$ at each iteration of the algorithm as a linear combination of the means $(\nu_{x,e_{i,j}})_{i\leq n, j\leq r}$ instead of a full new computation which can be heavy due to the presence of matrix exponentials. In our implementation, we used Scipy's minimize function with nonlinear constraints \footnote{\href{https://docs.scipy.org/doc/scipy/reference/generated/scipy.optimize.minimize.html}{ https://docs.scipy.org/doc/scipy/reference/generated/scipy.optimize.minimize.html}} which is a gradient descent algorithm with suitably adjusted step-sizes (based on the Sequential Least Squares Programming algorithm detailed in \cite{Kraft}). Due to the non-convexity of the problem, the algorithm is run a few times (typically $5$ times) with random initialisations of the parameters in order to avoid bad local minima. The solution with the smallest loss function is then chosen.\\

	The experiments consist of three parts:
	\begin{enumerate}
		\item first we check the accuracy achieved on a fixed-size testing set (30\% of the total available data) while increasing the size of the training set. 
		%Recall that the labels are randomly generated by the networks. More specifically, for a path $x$, its generated label is given by $v=\mathrm{sign}(\inner{y_u(T,x)}{\omega}+b)$, where $y_u(T,x)\sim \mathcal{N}\left(\nu_{x,u}, A \right)$. 
		Recall that the computed accuracies are random realisations depending on the results of the simulations (of the solution to the SDE (\ref{eq:MainSDE})).
		\item We verify the validity of the theoretical bound obtained in Theorem \ref{thm:GenErrSto} as a function of the size of the training set. More precisely, having computed some values for the parameters $u$, $\omega$ and $b$ and thus the corresponding hypothesis $H$, we compare the difference $\left|R(H)-\widehat{R}_{(X,V)}(H)\right|$	with the theoretical bound (which is valid with high probability)
	\[
	\frac{4}{\sqrt{2\pi m \lambda_{\min}(A)}} 
	\left(\Theta + 
	R\left(\int_0^t\left\|e^{W_0(t-s)}\right\|^2\mathrm{d}s\right)^{1/2} \right)
	 +\frac{2+5\sqrt{\log(2/\delta)/2}}{\sqrt{m}}.
	\]
	In the above, the true risk $R(H)$ is computed as an empirical risk over the testing set. The value $R$ is computed as the maximum of the $L^2$-norm of the paths in the whole dataset while $\Theta$ is computed as the largest value obtained for $|b|$ in the accuracy experiment detailed above. We take the value $\delta= 0.01$ so that the bound holds with probability $0.99$.
	\item We verify the robustness of the learning in the following way. We train the network with a corrupt dataset where a portion of the labels of the training inputs has been changed, then compute the resulting accuracy based on the test set. This robustness test is quite simple and does not show the full potential of the architecture but we suspect that such robustness also holds when the input paths are themselves being perturbed. For other measures of robustness, we refer the reader to the previously cited references and the additional references \cite{BNL, BNSJT, DDSV} that look into the robustness of training with corrupted datasets.
	\end{enumerate}	

	\begin{figure}[t]
	\centering
	\includegraphics[width=4.5cm, height=4.55cm]{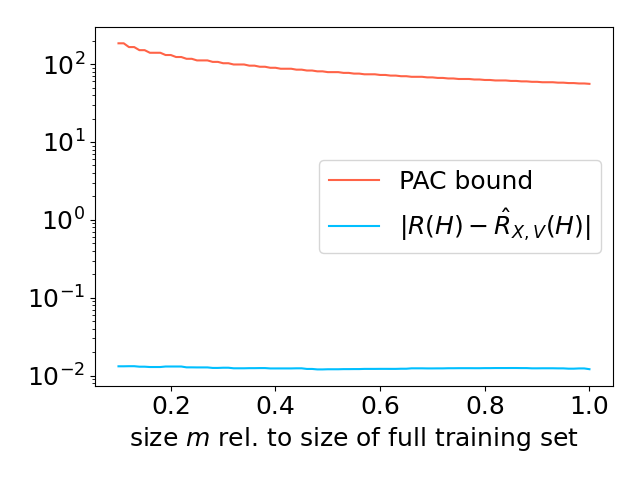}
	\includegraphics[scale=0.45]{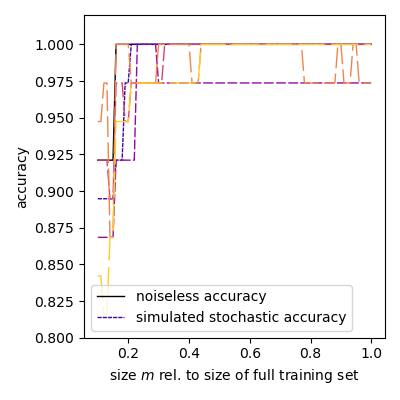}
	\includegraphics[scale=0.45]{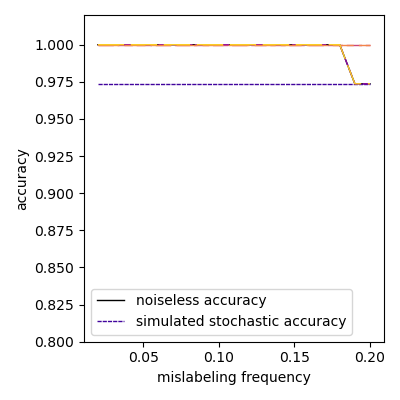}
	\\[\smallskipamount]
	\includegraphics[width=4.5cm, height=4.55cm]{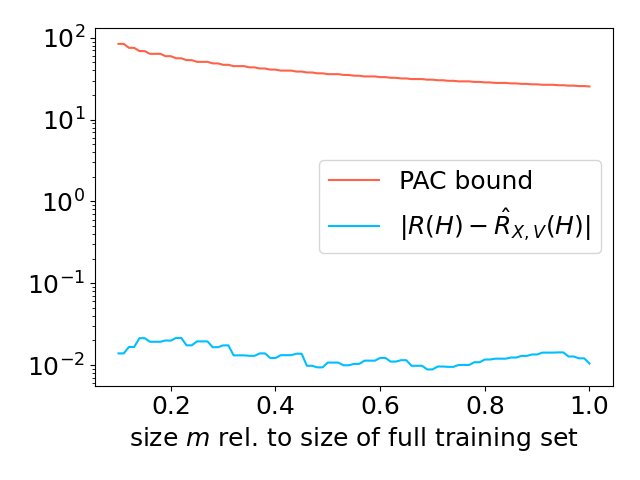}
	\includegraphics[scale=0.45]{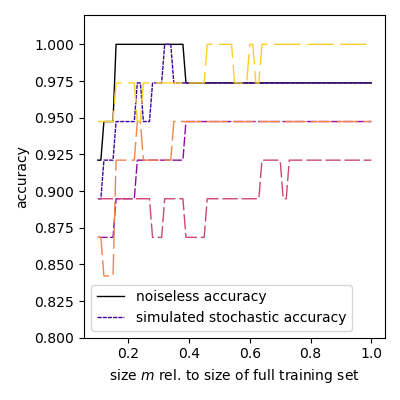}
	\includegraphics[scale=0.45]{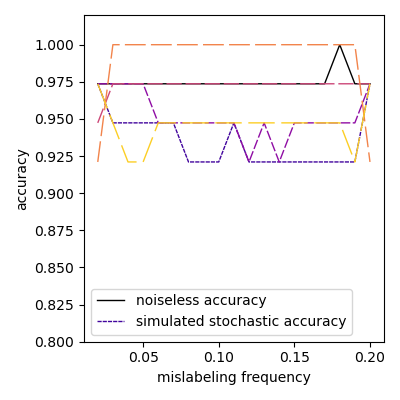}
	\\[\smallskipamount]
	\includegraphics[width=4.5cm, height=4.55cm]{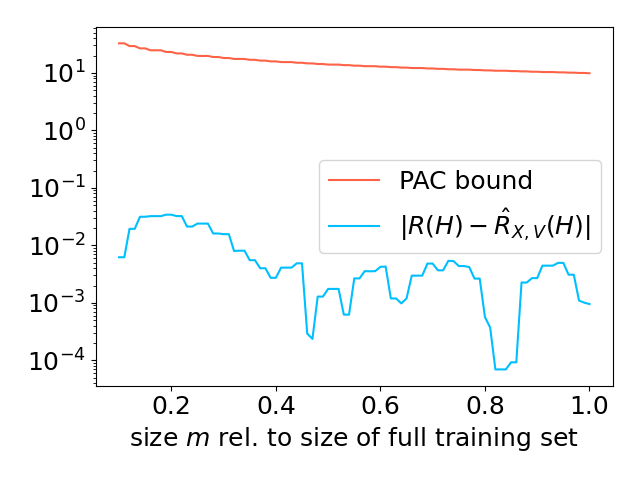}
	\includegraphics[scale=0.45]{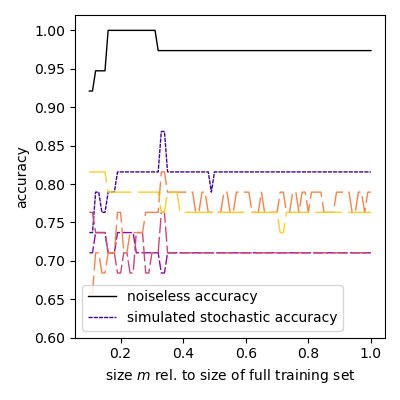}
	\includegraphics[scale=0.45]{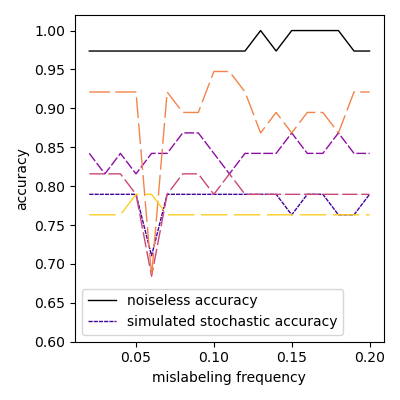}
	\caption{Verification of the theoretical PAC-bound (Theorem \ref{thm:GenErrSto}) and accuracy of the network for the Japanese vowels dataset. Each row of figures represents the results for different noise scales (as defined in Equation (\ref{Eq:NoiseMatrixFormula})); 1, 1.5 and 2.5 respectively. The middle column shows the evolution of the accuracy of the network when increasing the size of the training dataset. The right column shows how an increasing portion (described as mislabelling frequency) of corrupt labels in the training dataset affects the accuracy on the test set. Solid lines show the accuracy of classification by a noiseless RNN (trained as a stochastic RNN). Dotted lines show the results of simulations of the realised labels by the stochastic RNN.}
	\label{Fig:JapVowPlot}
	\end{figure}

	We start first by the experiments on the Japanese vowels dataset with three choices of noise scales $\delta\in \{1,1.5,2.5\}$ in (\ref{Eq:NoiseMatrixFormula}). The resulting plots are shown in Figure \ref{Fig:JapVowPlot}. In the accuracy plots:
	\begin{itemize}
	\item the solid lines show the accuracy based on classifying the averages $\nu_{x,u}$ of the paths in the test dataset with the obtained hyperplane $h_{\omega,b}$. In other words, the generated label for a path $x$ is $v=\mathrm{sign}(\inner{\nu_{x,u}}{\omega}+b)$.  As $\nu_{x,u}$ is the hidden state of a noiseless RNN\footnote{ Note however that a noiseless RNN will train differently and consequently produce different parameters $u$, $\omega$ and $b$ in general.}, we refer to this scenario as noiseless accuracy or NRNN.
	\item The dotted lines show the accuracy of 5 simulated labels produced by the stochastic RNN. Here, the generated label for a path $x$ is $v=\mathrm{sign}(\inner{y_u(T,x)}{\omega}+b)$ where $y_u(T,x)$ is simulated as $y_u(T,x)\sim \mathcal{N}\left(\nu_{x,u}, A \right)$. We refer to this scenario as simulated stochastic accuracy.
	\end{itemize}

	\begin{table}[t]
	\centering
	\begin{tabular}{|c|c||c|c|c|c||c|c|c|c|}
 	\hline
 	&\multicolumn{4}{|c|}{Accuracy}&\multicolumn{4}{|c|}{Robustness test accuracies}&\\
	 \hline
	 Noise scale & NRNN & min & max & avg. & NRNN & min & max & avg. & Ratio\\
	 \hline  
1&	100\% &	97.37\% &	100\% &	99.21\% &	100\% &	94.74\% &	100\% &	99.21\%& 100\%\\ \hline 
1.5	& 97.37\% &	86.84\% &	100\% &	95.00\% &	97.37\% &	92.11\% &	100\% &	95.79\%& 100.83\%\\ \hline 
2.5&	97.37\% &	65.79\% &	86.84\% &	75.26\% &	97.37\% &	76.32\% &	94.74\% &	83.95\%& 111.54\%\\ \hline 
	\end{tabular}
	\caption{Numerical values when working the Japanese vowels dataset. The first column shows the noise scale (see Equation (\ref{Eq:NoiseMatrixFormula})), the next four columns give respectively the accuracy of classification by a noiseless RNN (trained as a stochastic RNN), then the lowest, the largest and the average accuracy given by the (stochastic) RNN over 10 simulations. The following four columns show the same numbers after training the RNN with 10\% of mislabelled data. The last column shows the ratio of the average accuracy in the robustness test (after training with some corrupted labels) over the normal average accuracy (after training with the correct labels).}
	\label{Table:JapVow}
	\end{table}
		
	 Note that because of the discontinuity of the hyperplane classifying function, the former is not expected to be an average of the latter. For readability, we report some numerical values for 10 simulations in Table \ref{Table:JapVow}:
	 \begin{itemize}
	 \item the first column refers to the value of the noise scale $\delta$ in (\ref{Eq:NoiseMatrixFormula}),
	 \item the group of columns ``Accuracy'' reports the results of the experiment after training with the entire training set while the group of columns ``Robustness test accuracies'' reports the accuracy of the network after training with the entire training set with 10\% of mislabelled data,
	 \item the column NRNN gives the accuracy of classification of the averages $\nu_{x,u}$ with the hyperplane $h_{\omega,b}$ (these labels are given by as $v=\mathrm{sign}(\inner{\nu_{x,u}}{\omega}+b)$),
	 \item the columns ``min'', ``max'' and ``avg.'' give, respectively, the smallest, the largest and the average observed accuracy over the 10 simulations (the labels are simulated as $v\sim\mathrm{sign}(\inner{\mathcal{N}\left(\nu_{x,u}, A \right)}{\omega}+b)$.)
	 \end{itemize}
	 As usually observed when working with PAC-bounds based on the Rademacher complexity, the theoretical bound obtained in Theorem \ref{thm:GenErrSto} can be very loose. We note however in Figure \ref{Fig:JapVowPlot} that the difference gets smaller with larger noise scales. In Table \ref{Table:JapVow}, we note an almost perfect noiseless accuracy even in the presence of noise. A striking fact, however, is that, on average and in the presence of noise, there is no remarkable loss in the performance of the RNN when training with some mislabelled data.\\

		\begin{figure}[t]
	\centering
	\includegraphics[width=4.5cm, height=4.55cm]{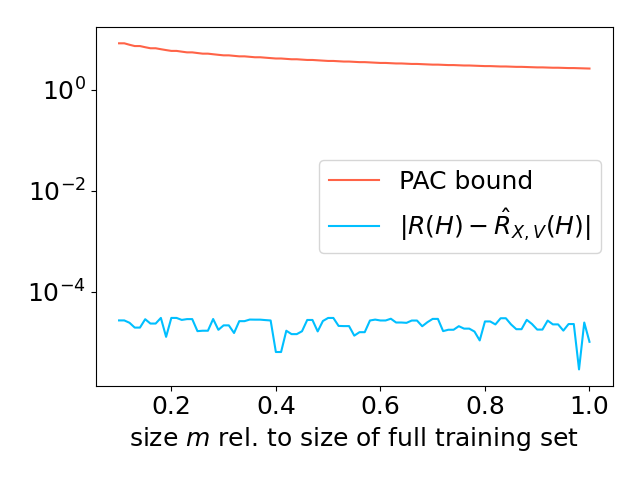}
	\includegraphics[scale=0.45]{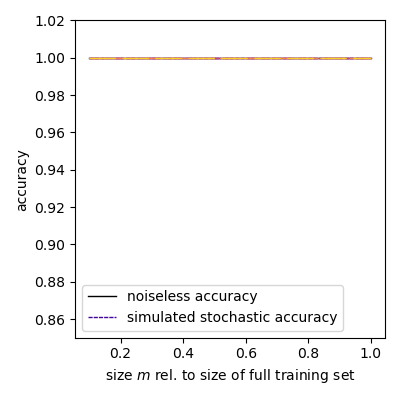}
	\includegraphics[scale=0.45]{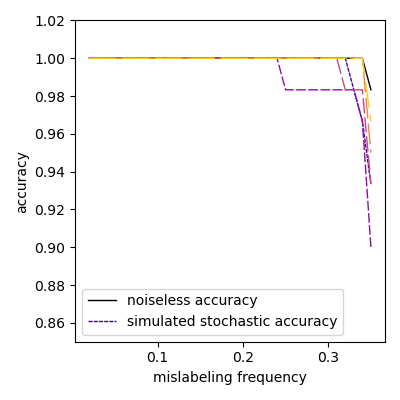}
	\\[\smallskipamount]
	\includegraphics[width=4.5cm, height=4.55cm]{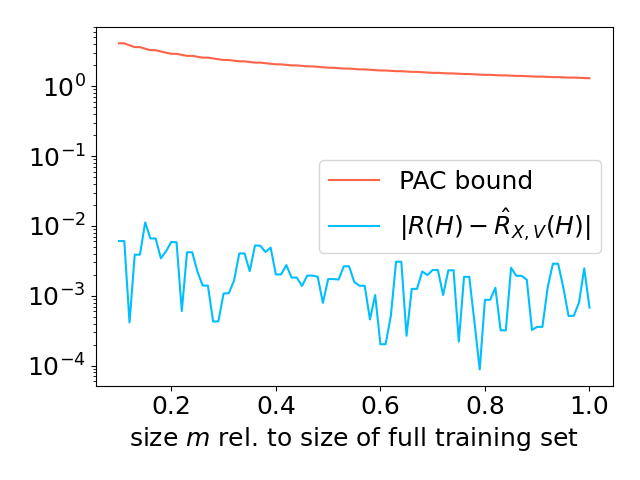}
	\includegraphics[scale=0.45]{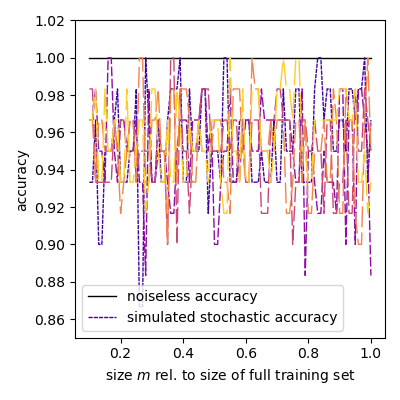}
	\includegraphics[scale=0.45]{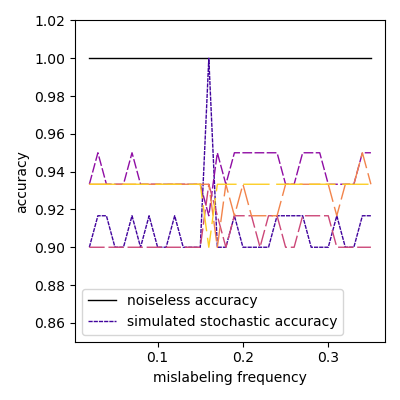}
	\\[\smallskipamount]
	\includegraphics[width=4.5cm, height=4.55cm]{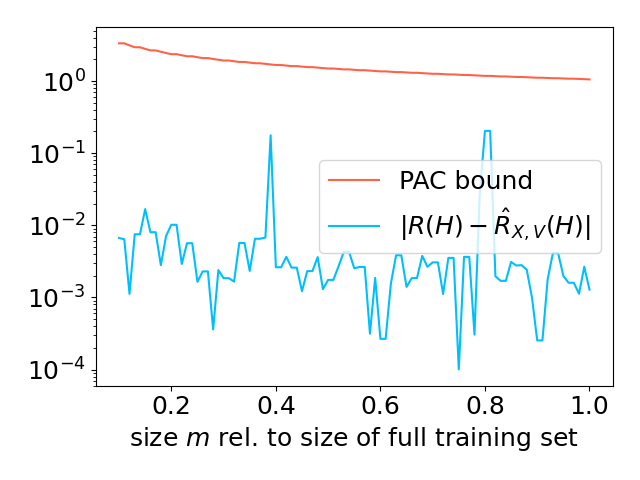}
	\includegraphics[scale=0.45]{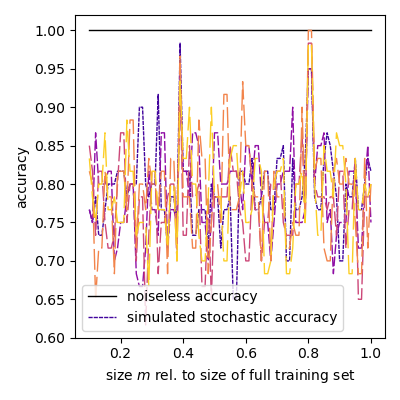}
	\includegraphics[scale=0.45]{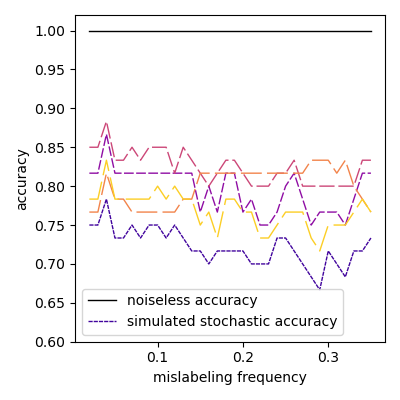}	
	\caption{Verification of the theoretical PAC-bound (Theorem \ref{thm:GenErrSto}) and accuracy of the network for the synthetic dataset. Each row of figures represents the results for different noise scales (as defined in Equation (\ref{Eq:NoiseMatrixFormula})); 2, 4 and 6 respectively. The middle column shows the evolution of the accuracy of the network when increasing the size of the training dataset. The right column shows how an increasing portion (described as mislabeling frequency) of corrupt labels in the training dataset affects the accuracy on the test set. Solid lines show the accuracy of classification by a noiseless RNN (trained as a stochastic RNN). Dotted lines show the results of simulations of the realised labels by the stochastic RNN.}
	\label{Fig:SinusPlot}
	\end{figure}

	\begin{table}[t]
	\centering
	\begin{tabular}{|c|c||c|c|c|c||c|c|c|c|}
 	\hline
 	&\multicolumn{4}{|c|}{Accuracy}&\multicolumn{4}{|c|}{Robustness test accuracies}&\\
	 \hline
	 Noise scale & NRNN & min & max & avg. & NRNN & min & max & avg. & Ratio\\
	 \hline  
2 &	100\% &	100\% &	100\% &	100\% &	100\% &	100\% &	100\% &	100\% &	100.00\% \\
	 \hline
4 &	100\% &	88.33\% &	98.33\% &	94.83\% &	100\% &	90.00\% &	100.00\% &	92.83\% &	97.89\% \\
	 \hline
6 &	100\% &	73.33\% &	90.00\% &	79.83\% &	100\% &	71.67\% &	81.67\% &	76.67\% &	96.03\% \\
	 \hline
	\end{tabular}
	\caption{Numerical values when working the synthetic dataset. The first column shows the noise scale (see Equation (\ref{Eq:NoiseMatrixFormula})), the next four columns give respectively the accuracy of classification by a noiseless RNN (trained as a stochastic RNN), then the lowest, the largest and the average accuracy given by the (stochastic) RNN over 10 simulations. The following four columns show the same numbers after training the RNN with 10\% of mislabelled data. The last column shows the ratio of the average accuracy in the robustness test (after training with some corrupted labels) over the normal average accuracy (after training with the correct labels).}
	\label{Table:Sinus}
	\end{table}

	Training with synthetic data showed a higher tolerance to the noise scales. In Figure \ref{Fig:SinusPlot}, we show the result of the same experiments as before with the noise scales $\delta\in \{2,4, 6\}$. Numerical values are reported on Table \ref{Table:Sinus} (with 15\% for mislabelled data in the robustness test). We observe again the validity of the theoretical PAC-bound obtained in Theorem \ref{thm:GenErrSto} (with the difference getting smaller with higher noise scales) and notice the same persistence in the quality of the labels generated by the RNN when training with mislabelled data.
		%%%%%
	%% Bibliography
	%%%%%
%	\clearpage
	\bibliographystyle{alpha}
	\bibliography{bibLinRes}
\end{document}